\documentclass{article} 
\usepackage{iclr,times}

\usepackage[utf8]{inputenc} 
\usepackage[T1]{fontenc}    
\usepackage{hyperref}       
\usepackage{url}            
\usepackage{booktabs}       
\usepackage{amsfonts}       
\usepackage{nicefrac}       
\usepackage{microtype}      
\usepackage{xcolor}         
\usepackage{float}
\usepackage{amsmath}
\usepackage{amssymb}
\usepackage{mathtools}
\usepackage{amsthm} 
\usepackage{thm-restate}
\usepackage{longtable,booktabs,tabularx, enumitem}

\newtheorem{theorem}{Theorem}
\newtheorem{proposition}[theorem]{Proposition}

\newtheorem{remark}[theorem]{Remark}

\usepackage{mathnotations}
\usepackage{graphicx}
\usepackage{wrapfig}
\usepackage{subcaption}
\usepackage{caption}
\usepackage{tikz}
\usetikzlibrary{arrows.meta,decorations.markings,calc,intersections,backgrounds, shapes.geometric}

\definecolor{darkgreen}{rgb}{0,0.5,0}

\usepackage[textsize=tiny]{todonotes}

\title{Why DPO is a Misspecified Estimator and \\How to Fix It 
}

\author{Aditya Gopalan\\
IISc Bangalore\\
\texttt{aditya@iisc.ac.in} \\
\And
Sayak Ray Chowdhury \\
IIT Kanpur\\
\texttt{sayakrc@iitk.ac.in} \\
\And
Debangshu Banerjee \\
HP AI Research\\
\texttt{debangshu.banerjee@hp.com}
}

\allowdisplaybreaks

\iclrfinalcopy 
\begin{document}

\maketitle

\begin{abstract}
Direct alignment algorithms such as Direct Preference Optimization (DPO) fine-tune models based on preference data, using only supervised learning instead of two-stage reinforcement learning with human feedback (RLHF). We show that DPO encodes a statistical estimation problem over reward functions induced by a parametric policy class. When the true reward function that generates preferences cannot be realized via the policy class, DPO becomes misspecified, resulting in failure modes such as preference order reversal, worsening of policy reward, and high sensitivity to the input preference data distribution. On the other hand, we study the local behavior of two-stage RLHF for a parametric class and relate it to a natural gradient step in policy space.  Our fine-grained geometric characterization allows us to propose AuxDPO, which introduces additional auxiliary variables in the DPO loss function to help move towards the RLHF solution in a principled manner and mitigate the misspecification in DPO. We empirically demonstrate the superior performance of AuxDPO on didactic bandit settings as well as LLM alignment tasks.

\end{abstract}

\section{Introduction}
\label{sec:intro}

Preference-based alignment is a key part of the training process of large language models (LLMs). It aims to steer a pretrained model’s conditional distribution toward outputs that humans (or carefully calibrated annotator models) prefer. Formally, given comparison data $(s, a_w, a_l)$, the goal is to shape a policy $\pi$ whose induced responses align with a latent reward model that generated those preferences.

Two-stage RLHF is the standard way of carrying out preference-based alignment \citep{ziegler2019fine}. However, it is computationally demanding (it requires training a separate reward model) and complex due to a two-stage pipeline (supervised learning for the reward model followed by RL policy optimization based on the learned reward model). Concretely, the reward model $r_\phi(s,a)$ is trained on preference pairs via a Bradley–Terry/Logistic objective~\citep{bradley1952rank}, maximizing $\log \sigma\!\big(r_\phi(s,a_w) - r_\phi(s,a_l)\big)$ over $(s,a_w,a_l)$. The second stage then optimizes a KL-regularized objective of the form $\max_{\pi}\; \mathbb{E}_{s \sim \rho,\, a \sim \pi(\cdot\mid s)}\!\big[r_\phi(s,a)\big]\;-\;\beta\,D_\KL\!\big(\pi(\cdot\mid s)\,\|\,\pi_{\mathrm{ref}}(\cdot\mid s)\big)$,
typically implemented with PPO-style updates. This stage is on-policy and rollout-heavy: the model must repeatedly generate samples to estimate advantages under $r_\phi$, maintain a stable KL to the reference policy $\pi_{\mathrm{ref}}$ (often the SFT model), and tune sensitive hyperparameters (e.g., $\beta$, clip ranges, learning rates). In practice, this entails nontrivial engineering (reward hacking mitigation, variance reduction, response-length control) and significant compute for both reward-model training and RL updates, which motivates interest in lighter-weight alternatives.

The introduction of direct alignment algorithms, such as Direct Preference Optimization (DPO) \citep{rafailov2023direct}, was a landmark step that paved the way for lightweight alignment of a base model using preference data and a single supervised training phase. DPO operates by explicitly solving the second, KL-regularized, policy optimization phase of RLHF and using it to reparameterize the first phase of reward learning in terms of the optimized policy, in effect achieving a one-step equivalent to the original two-step pipeline. This has been instrumental in enabling both industrial players and the open-source AI community to carry out fast alignment of models without the burden of additional resources. Many variants of DPO have since been developed, catering to various aspects of direct alignment.

Despite its widespread appeal, however, the design of DPO rests on the idealized assumption that the policy class is {\em tabular}, i.e., it includes every possible input-output conditional probability distribution $(\pi(a\mid s))_{s,a}$, where $s$ and $a$ denote prompt and response strings, respectively. This assumption enables the KL-regularized policy optimization problem to be solved in closed form and used explicitly to derive the equivalent supervised DPO loss \cite[Appendix A.1]{rafailov2023direct}.

In contrast, real-world LLMs are far from tabular and are, in fact, parametric policy classes that naturally result from the use of neural architectures (e.g., Transformers) with a finite number of parameters. One may then ask: Does minimizing the DPO loss over a non-tabular policy class still preserve the claimed equivalence with full two-stage RLHF? If not, then how does it differ from the ideal RLHF-optimal policy? Does it enjoy any guarantees with respect to the performance of the latter, and, if not, is there a principled fix? 

\begin{wrapfigure}{r}{0.65\textwidth}
\vspace{-10pt}

    \resizebox{\linewidth}{!}{%
\begin{tikzpicture}[scale=1.5]
  \tikzset{
    blueline/.style={thick, blue!70!blue, opacity=0.9},
    redline/.style={thick, red!70!red, opacity=0.9},
    brownline/.style={thick, brown!70!brown, opacity=0.9},
    yellowhighlight/.style={line width=3mm, yellow!80!orange, opacity=0.4},
    orangehighlight/.style={line width=3mm, orange!80!orange, opacity=0.4},
    greenpoint/.style={fill=green!70!black, circle, inner sep=1.5pt},
    bluepoint/.style={fill=blue!70!cyan, circle, inner sep=1.5pt, text=blue!70!cyan},
    blackpoint/.style={fill=black!70!black, circle, inner sep=1.5pt, text=blue!70!cyan},
    yellowpoint/.style={fill=yellow!80!orange, circle, inner sep=1.5pt},
    redpoint/.style={fill=red!80!red, circle, inner sep=1.5pt},
    orangepoint/.style={fill=orange!80!red, circle, inner sep=1.5pt},
    bluestar/.style={star, star points=5, star point ratio=2.5, fill=blue!70!cyan, inner sep=0.5pt, minimum size=10pt},
    yellowstar/.style={star, star points=5, star point ratio=2.5, fill=yellow!80!orange, inner sep=0.5pt, minimum size=10pt},
    graytext/.style={text=gray!70!black},
  }
  
  \coordinate (SW) at (-5, -3);
  \coordinate (SE) at (5, -3);
  \coordinate (NW) at (-5, 3);
  \coordinate (NE) at (5, 3);
  
  

  \draw[thick, green!70!black] (-4.75, 2.25) .. controls (-1.75, -0.25) .. (-1.5, -2.75);

  \coordinate (RTheta0) at (-2.1, -0.25);
  \coordinate (Rstar) at ($(RTheta0) + (60:1.5)$);
  
  \fill[green!70!black] (RTheta0) circle (1.5pt);
  \fill[black] (Rstar) circle (1.5pt);

  \draw[semithick, gray!60] ($(RTheta0) + (-0.5, -0.4)$) rectangle ($(Rstar) + (0.3, 0.3)$);
  \node[black, left=2pt] at (Rstar) {$r^*$};
  \node[green!70!black, below left=2pt] at (RTheta0) {$r_{\theta_0}$};

  \coordinate (RProj1) at ($(RTheta0) + (-0.4, 0.5)$);
  \coordinate (RProj2) at ($(RTheta0) + (0.18, -0.3)$);
  
  \fill[orange!80!red] (RProj1) circle (1pt);
  \fill[orange!80!red] (RProj2) circle (1pt);
  
  \draw[-{stealth[scale=2]}, thick, orange!80!red] (Rstar) -- (RProj1);
  \draw[-{stealth[scale=2]}, thick, orange!80!red] (Rstar) -- (RProj2);

  \draw[semithick, gray!60] (0.25, -2.5) rectangle (4.5, 2.5);

  \coordinate (RTheta0_zoom) at (2, -1.2);
  \fill[green!70!black] (RTheta0_zoom) circle (1.5pt);
  \node[green!70!black, below left=2pt] at (RTheta0_zoom) {$r_{\theta_0}$};

  \coordinate (Rstar_zoom) at ($(RTheta0_zoom) + (60:3.8)$);
  \fill[black] (Rstar_zoom) circle (1.5pt);
  \node[black, right=2pt] at (Rstar_zoom) {$r^*$};

  \draw[thick, green!70!black] ($(RTheta0_zoom) + 1.05*(-1.5, 2.5)$) -- ($(RTheta0_zoom) + 0.45*(1.5, -2.5)$);

  \coordinate (RThetaRLHF) at ($(RTheta0_zoom) + 0.33*(-1.5, 2.5)$);
  \fill[blue!70!cyan] (RThetaRLHF) circle (1.5pt);
  \node[blue!70!cyan, left=2pt] at (RThetaRLHF) {$r_{\theta_{\text{RLHF}}}$};

  \draw[dashed, blue!70!cyan] ($(RThetaRLHF) + 0.3*($(RThetaRLHF) - (Rstar_zoom)$)$) -- ($(Rstar_zoom) + 0.1*($(Rstar_zoom) - (RThetaRLHF)$)$);

  \draw[thin, gray!70, dashed] ($(Rstar) + (0.3, 0.3)$) -- (0.25, 2.5);
  \draw[thin, gray!70, dashed] ($(Rstar) + (0.3, -1.7)$) -- (0.25, -2.5);

  \node[orange!80!red, align=center, font=\footnotesize] at (-3.5, -0.5) {DPO solutions\\based on\\preference data\\distribution};

  \node[green!70!black, align=center, font=\footnotesize] at (-2.7, -2.5) {DPO's implicit reward\\manifold};

  \node[blue!70!cyan, right=8pt, font=\footnotesize] at (RThetaRLHF) {Ideal RLHF solution};

  \node[green!70!black, align=center, font=\footnotesize] at (3.5, -2.2) {Locally linearized\\manifold};

  \node[black, above=18pt, align=center, font=\footnotesize] at (Rstar) {True\\reward function};

\coordinate (midpoint) at ($(RThetaRLHF)!0.45!(Rstar_zoom)$);
\coordinate (midpoint_above) at ($(midpoint) + (0, 0.5)$);
\coordinate (RThetaRLHF_above) at ($(RThetaRLHF) + (0, 0.5)$);
\coordinate (Rstar_zoom_above) at ($(Rstar_zoom) + (0, 0.5)$);
\node[blue!70!cyan, font=\footnotesize, sloped, rotate=47, align=left] at (midpoint_above) {AuxDPO adds\\degrees of freedom};

\draw[-stealth, semithick, blue!70!cyan] ($(midpoint_above) + 0.3*($(RThetaRLHF_above) - (Rstar_zoom_above)$)$) -- 
                             ($(RThetaRLHF_above) + -0.05*($(midpoint_above) - (RThetaRLHF_above)$)$);
\draw[-stealth, semithick, blue!70!cyan] ($(midpoint_above) + 0.3*($(Rstar_zoom_above) - (RThetaRLHF_above)$)$) -- 
                             ($(Rstar_zoom_above) + 0.15*($(midpoint_above) - (Rstar_zoom_above)$)$);

\end{tikzpicture}
}
    \caption{{\footnotesize The geometry of DPO for parametric policies. (Left) DPO essentially performs a projection of the true preference-generating reward function ($r^*$ in black) onto the manifold of reward functions implicitly expressed by the policy class. If $r^*$ is in the manifold, then DPO finds the correct KL-regularized RLHF policy, but otherwise, the policy found (any orange point) is unreliable. (Right, zoomed inset) Locally linearizing the manifold around the base policy's implicit reward function ($r_{\theta_0}$) uncovers geometric insights. To reliably drive the solution to the reward function corresponding to the ideal RLHF solution ($r_{\theta_{\text{RLHF}}}$ in blue), AuxDPO introduces additional controlled degrees of freedom, along the null space of a base-policy dependent matrix to sidestep misspecification.}}
    \label{fig:global}
\vspace{-10pt}
\end{wrapfigure}
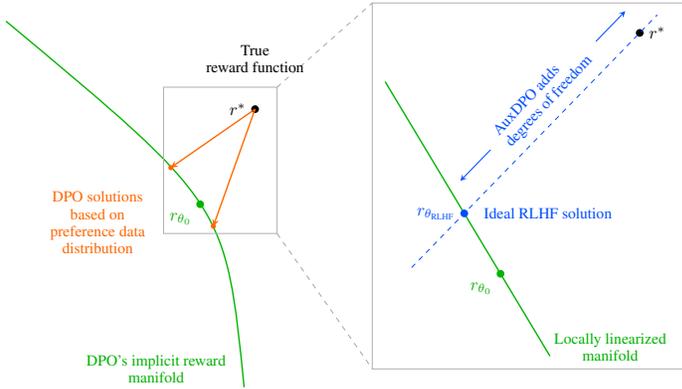

We address these questions by introducing a systematic framework to uncover the geometry of direct preference optimization in parametric policy classes. Our study helps show how DPO essentially solves a misspecified statistical estimation problem in the space of reward functions that are implicitly parameterized by the underlying policy class. Misspecified estimation problems have been analyzed in the statistics literature for exhibiting undesirable phenomena, such as inconsistent and arbitrary estimates that are sensitive to the input data distribution \citep{white1982maximum}. We show that such phenomena also manifest in the DPO setting. Our analysis framework also enables us to modify DPO in a principled manner, yielding a new algorithm (AuxDPO) that achieves the performance of two-stage RLHF in parametric models. More specifically, we make the following contributions:

\begin{enumerate}[leftmargin=1em]
    \item We show that for general parametric policy classes, there is a misspecified statistical estimation problem at the core of the DPO algorithm by design: DPO loss minimization is equivalent to a weighted KL-projection of the true reward function $r^*$ onto the (parametric, lower-dimensional) manifold of reward functions induced by the policy class. The weights of the projection are governed by the preference data collection frequencies (Fig. \ref{fig:global}, left).

    \item  We show that DPO, in the misspecified setting, can suffer from various failure modes such as order reversal of preferences, overall reward reduction, sensitivity to preference data frequencies, etc. These failure modes occur even with `clean data', i.e., infinite preference data generated using a BTL model based on an underlying true reward function $r^*$ and fed to DPO. Our analysis is based on taking a local, linearized view of DPO's implicit reward function manifold, which is accurate in the large-$\beta$ regime.

    \item 
    On the other hand, studying the local geometry of two-stage RLHF for general parametric policy classes yields new insights about linear equivalence classes of reward functions. We use these insights to design AuxDPO, a new direct preference optimization algorithm that effectively mitigates the misspecification issue by introducing auxiliary controlled degress of freedom in  reward space (Fig. \ref{fig:global}, right). We demonstrate the effectiveness of AuxDPO in experiments. On real-world LLM preference tuning tasks, AuxDPO consistently outperforms DPO in aligning to held-out human preferences, confirming its practical value.
\end{enumerate}

{\bf Related work.} A recent line of work focuses on studying the insufficiency and implications of the tabular policy class assumption. \citet{gao2024rebel} and \citet{swamy2025all} call into question the tabular policy class assumption in the context of original two-stage RLHF. \citet{tajwar2024preference} carry out an empirical investigation and note that the standard DPO loss can inadvertently reduce the model’s absolute likelihood of chosen responses as long as the relative probability between chosen and rejected responses increases. \citet{meng2024simpo} and \citet{xu2024contrastive} propose fixes based on considerations of margin and length normalization, and elimination of the reference policy. 

\citet{xu2024dpo} and \citet{song2024importance} study the shortcomings of DPO arising from a lack of coverage, arguing that DPO can fail if a strong coverage condition is not met. The latter provides a counterexample to this end, showing the existence of an implicitly expressible reward function that is $\epsilon$-approximately close to the true reward function but corresponds to a policy not in the KL neighborhood of the base policy. It is, however, unclear if such an implicit reward function can actually be output by DPO. 
Our fine-grained analysis in this paper shows that even with perfect coverage (uniform base policy), the policy returned by DPO can suffer from pathologies such as preference reordering and a decrease of overall expected reward  (Proposition \ref{prop:toy}). 

A separate line of work focuses on the gradient dynamics of DPO loss optimization and its impact on policy probabilities \citep{pal2024smaug, razin2024unintentional, jian2025stable}. It is shown that an individual gradient step on the standard DPO loss can result in likelihood displacement, where the probability of preferred responses can drop relative to the base policy for a gradient step. Our approach eschews assuming any specific optimization algorithm such as gradient descent and considering individual gradient steps, and instead focuses on showing failure modes such as likelihood displacement, preference reversal, reward reduction, etc. by studying the minimizer of the DPO loss. 

\citet{shi2025understanding}, perhaps the closest in spirit to our study, demonstrates a multi-armed bandit example that, when subjected to DPO with a log-linear policy class, does not cause any movement from the base policy. However, this example relies on a degenerate and symmetric reward function; we are able to demonstrate, in a fine-grained manner, that the policy can strictly worsen from the base policy in a manner that is highly sensitive to the preference data distribution (Proposition \ref{prop:toy}).

\section{Preliminaries}
\label{sec:prelims}

Let $\cD=\left(s^{(i)},a_w^{(i)},a_l^{(i)}\right)_{i=1}^{n}$ be a dataset of $n$ samples, where each sample has a prompt $s \in \cS$, two responses $a_w, a_l \in \cA$ such that $a_w \succ a_l$, i.e., $a_w$ is a preferred response over $a_l$. We assume both $\cS$ and $\cA$ are finite sets, with $\abs{\cS}\cdot\abs{\cA}=m$. The prompt $s$ is sampled from a distribution $\rho$ over $\cS$. The pair of responses $(a_w, a_l)$ is sampled from some base (reference) policy $\pi_{\text{ref}}$ conditioned on $s$, i.e., $a_w, a_l \sim \pi_{\text{ref}}(\cdot|s)$. The preference ordering between a pair of responses is assumed to be sampled according to a Bradley-Terry-Luce (BTL) model: the probability of $a$ being preferred to $a'$ is given by  $p_{s,a,a'}^{\BTL}(r^*) = \sigma\!\left(r^*(s,a)- r^*(s,a') \right)$, where $r^*: \cS \times \cA \to \Real $ is a (latent) reward function and $\sigma(z) := \frac{1}{1+ e^{-z}}$ is the sigmoid function.

Let $\pi_\theta: \cS \to \Delta(\cA)$ be a policy (e.g., a language model) smoothly parameterized by a $d$-dimensional vector $\theta \in  \Real^d$ (e.g., the weights of a transformer), where $\Delta(\cA)$ denotes the probability simplex over $\cA$. Let $\theta_0\in \Real^d$ be the parameter for the base policy $\pi_{\text{ref}}$ so that $\pi_{\text{ref}}=\pi_{\theta_0}$. A special case is the tabular policy class, where $d=m=|S|\cdot|A|$ and $\pi_\theta(a|s) =\theta_{s,a}$ (assuming w.l.o.g. that $\sum_{s,a} \theta_{s,a}=1$). However, LLM policy classes are structured and non-tabular with parameter dimension $d \ll m$, e.g., the neural softmax policy $\pi_\theta(a\mid s) = \frac{\exp(f_\theta(s,a))}{\sum_{a'\in \cA}\exp(f_\theta(s,a'))}$,
where $f_\theta$ is, say, a neural network.

For a given reward function $r^*$, the optimal policy in a KL-regularized sense is obtained by maximizing the following objective:
\begin{align}\label{eq:RLHF}
J(\theta;r^*)=\mathbb{E}_{\rho, \pi_\theta}\left[r^*(s,a)-\beta\log\frac{\pi_\theta(a|s)}{\pi_{\theta_0}(a|s)}\right] = \mathbb{E}_{\rho, \pi_\theta}\left[r^*(s,a) - \beta D_\KL(\pi_\theta(\cdot|s) ||\pi_{\theta_0}(\cdot|s)) \right],
\end{align}
where $\mathbb{E}_{\rho, \pi_\theta}$ denotes expectation taken over $s\sim\rho(\cdot)$ and $a \sim \pi_\theta(\cdot \mid s)$, and $\beta > 0$ is a parameter that controls the amount of deviation from the base policy. We assume that $\theta^* \in \Real^d$ is the unique minimizer of \eqref{eq:RLHF}. For our analytical results, we will focus on the `local' case $\beta \gg 1$, meaning that the policy is not allowed to move beyond a local neighborhood of $\pi_{\theta_0}$.
When the policy class is tabular, it follows that the optimal policy $\pi_{\theta^*}$ and the latent reward $r^*$ satisfy 
\begin{align}\label{eq:reward-policy-equiv}
    \pi_{\theta^*}(a|s) = \frac{1}{Z^*(s)}\pi_{\theta_0}(a|s)\exp(r^*(s,a)/\beta) \iff r^*(s,a) = \beta \log\frac{\pi_{\theta^*}(a|s)}{\pi_{\theta_0}(a|s)} + \beta \log Z^*(s)~,
\end{align}
where $Z^*(s)= \sum_{a\in \cA}\pi_{\theta_0}(a|s)\exp(r^*(s,a)/\beta)$ is the normalizing or partition function \citep{rafailov2023direct}. Under this reward-policy equivalence, the preference probabilities under the BTL model can be expressed using the optimal policy $\pi_{\theta^*}$ and the base policy $\pi_{\theta_0}$ as follows.
\begin{align*}
p_{s,a,a'}^{\text{BTL}}(r^*)= \sigma\!\left(\!\beta \log\!\frac{\pi_{\theta^*}(a|s)}{\pi_{\theta_0}(a|s)}\!-\!\beta \log\!\frac{\pi_{\theta^*}(a'|s)}{\pi_{\theta_0}(a'|s)}\!\right)= \sigma \left(r^\beta_{\theta^*}(s,a)-r^\beta_{\theta^*}(s,a') \right)~,
\end{align*}
where, for any $\theta \in \Theta$, $r^\beta_\theta: \cS \times \cA \to \Real$ defined via $r^\beta_\theta(s,a) := \beta \log\!\frac{\pi_{\theta}(a|s)}{\pi_{\theta_0}(a|s)}$ denotes the implicit reward function corresponding to the policy $\pi_\theta$ at deviation level $\beta$ (note that $r^\beta_{\theta_0} \equiv 0$ by definition). Let 
$\cR^\beta = \left\lbrace r^\beta_\theta: \theta \in \Real^d\right \rbrace \subsetneq \Real^m$ be the set of all implicit reward functions induced by the policy parameters $\theta$ at deviation level $\beta$. Given the dataset $\cD$, DPO \citep{rafailov2023direct} finds the minimizer of the {\em empirical} DPO loss
\begin{align*}
   \!\! \cL_{\cD}(\theta) \!=\! -\!\sum_{i=1}^n \!\log\sigma\!\left(\!r^\beta_\theta(s^{(i)},a_w^{(i)})\!-\!r^\beta_\theta(s^{(i)},a_l^{(i)})\!\right)\!= \!-\!\!\sum_{s,a_w, a_l} \!\!N_{s,a_w,a_l} \log\sigma\!\left(r^\beta_\theta(s,a_w)\!-\!r^\beta_\theta(s,a_l)\!\right),
\end{align*}
where $N_{s,a_w,a_l}$ is the total number of pairwise preferences for which $a_w \succ a_l$ at $s$. If $n_{s,a,a'}$ denotes the total number of pairwise preferences for the triplet $(s,a,a')$ (assumed non-random and fixed in advance), then $N_{s,a_w,a_l} \sim \texttt{Binomial}\left(n_{s,a_w,a_l}, p_{s,a_w,a_l}^{\BTL}(r^*)\right)$, yielding the {\em population} DPO loss
\begin{align}\label{eq:DPO-population}
   \cL(\theta)  = -\sum_{s,a,a'} n_{s,a,a'} \left[ p_{s,a,a'}^{\BTL}(r^*) \log p_{s,a,a'}^{\BTL}(r^\beta_\theta) + (1 - p_{s,a,a'}^{\BTL}(r^*)\log \big(1-p_{s,a,a'}^{\BTL}(r^\beta_\theta)\big) \right].
\end{align}
We take up this loss as our main object of study for DPO in the sequel.

\section{Reward Misspecification in DPO}

Let $r^*$ denote the $m$-dimensional vector of latent rewards $\left( r^*(s,a)\right)_{s,a}$, by slightly abusing notation.

\begin{restatable}[DPO is weighted KL-projection]{proposition}{est}
\label{prop:est}
Assume that the pairwise preference data are drawn from  $p_{s,a_w,a_l}^{\BTL}(r^*)$ for some $r^* \in \mathbb{R}^m$, with $n_{s,a,a'}$ preference pairs drawn for each triplet $(s,a,a')$. If $\theta_{\text{DPO}}$ minimizes the DPO loss \eqref{eq:DPO-population}, then its corresponding implicit reward function satisfies
\begin{align}\label{eq:reverse-kl-project}
r^\beta_{\theta_{\text{DPO}}} = \arg\min_{r \in \cR^\beta} \sum_{s, a,a'} n_{s,a,a'} \, d_{\KL}\big(p^{\BTL}_{s,a,a'}(r^*) || \,p^{\BTL}_{s,a,a'}(r)\big)~,
\end{align}
where $d_\KL(p||q)$ denotes the KL divergence b/w two Bernoulli random variables with parameters $p,q$.
\end{restatable}
The result establishes that DPO projects (according to reverse-KL divergence weighted by pairwise preference counts $n_{s,a,a'}$) the true reward function $r^*$ onto the set of implicit reward functions $\cR^\beta$ (with the corresponding policy being returned after fine-tuning). It implies that if $r^*$ is realizable, i.e., $r^* = r^\beta_\theta$ for some $\theta$, then this projection (trivially) finds $r^\beta_\theta$ and hence the policy $\pi_\theta$, which coincides with the RLHF policy, i.e., $\theta = \theta^*$. 

However, if $r^* \notin \cR^\beta$ (which is typically the case since $\cR^\beta$ is a lower-dimensional ($d$-dimensional) manifold of $\Real^m$), then we are in the misspecified estimation setting. In this case, the result of the KL-projection will, in general, be dependent on the exact weighted projection which is determined by the preference data frequencies $\left( n_{s,a,a'} \right)_{s,a,a'}$. We demonstrate, in the next section, that the policies resulting from DPO's misspecified estimation enjoy no guarantees: they could suffer from preference reversal, or even worse, yield a lower average reward than even the base policy (contrary to two-stage RLHF, where the average reward can never decrease).

\subsection{Local Geometry of DPO} 

Let us locally approximate the implicit reward $r^\beta_{\theta}(s,a)$ via its first-order Taylor expansion around $\theta_0$:
\begin{align*}
r^\beta_{\theta}(s,a) \approx r^\beta_{\theta_0}(s,a) 
+ \big\langle \nabla r_{\theta_0}(s,a), \, \theta - \theta_0 \big\rangle = \beta \big\langle \tfrac{\nabla \pi_{\theta_0}(a \mid s)}{\pi_{\theta_0}(a \mid s)}  ,  \theta - \theta_0 \big\rangle = \beta \nabla \log \pi_{\theta_0}(a|s)^\top (\theta-\theta_0)~.
\end{align*}
Define the $d \times m$
matrix $A_{\theta_0}$ as the matrix containing $\nabla \log \pi_{\theta_0}(a|s)$ in its $(s,a)$-th column. 
With this, the local linear approximation of $r^\beta_\theta$ takes the form $\overline r^\beta_\theta = \beta A_{\theta_0}^\top (\theta-\theta_0)$. Since
$\left\lbrace \overline r^\beta_\theta: \theta \in \Real^d\right \rbrace$ is the column space of $A_{\theta}^\top$, we arrive at the linear manifold approximation $\cR^\beta \approx \cC\big( A_{\theta_0}^{\top} \big)$, for $\theta$ in the local neighborhood of $\theta_0$. Note that this linear approximation is independent of the choice of $\beta$ and is a function only of the policy class and base policy $\pi_{\theta_0}$.
\begin{remark}  
    \label{rem:errorDPO}
    The error in the linear approximation of the manifold $\cR^\beta$ by $\cC\big( A_{\theta_0}^{\top} \big)$ can be controlled to within any arbitrary tolerance by taking the policy deviation parameter $\beta$ to be sufficiently large. Proposition~\ref{prop:error} formally controls the approximation error. In the sequel, we only work with the linear approximation to develop our results. 
\end{remark}

Armed with this local linearization of the implicit reward manifold, we now show an example of a single-prompt, 3-response setting, with a 1-dimensional policy parameter, in which DPO exhibits counterintuitive and unexpected behavior, including preference reversal, reward reduction, and high sensitivity to the pairwise preference data counts $n_{s,a,a'}$. 

\begin{proposition}[Example of DPO with preference reversal and reward decrease]
\label{prop:toy}
    There exists a \emph{promptless} policy optimization problem with three responses and linear softmax policy class parameterized with a 1-dimensional parameter $\theta$ and a true reward vector $r^* \in \mathbb{R}^3$ such that DPO, carried out with pairwise preferences generated according to BTL($r^*$), sufficiently large $\beta$ and pairwise counts $\{n_{i,j}\}$, yields a policy $\pi_\theta$ such that (i) $\pi_\theta$ favors the response with second highest reward, (ii) $\pi_\theta$ decreases (resp. increases) the probability of the action with the highest reward (resp. second highest reward) w.r.t. the base policy $\pi_{\theta_0}$, and (iii) $\pi_{\theta}^\top r^* < \pi_{\theta_0}^\top r^*$.
\end{proposition}

\begin{wrapfigure}[25]{r}{0.4\textwidth}
\vspace{-8pt}
    \resizebox{\linewidth}{!}{%
     \begin{tikzpicture}[scale=1.5]
      \tikzset{
        blueline/.style={thick, blue!70!blue, opacity=0.9},
        redline/.style={thick, red!70!red, opacity=0.9},
        brownline/.style={thick, brown!70!brown, opacity=0.9},
        yellowhighlight/.style={line width=3mm, yellow!80!orange, opacity=0.4},
        orangehighlight/.style={line width=3mm, orange!80!orange, opacity=0.4},
        greenpoint/.style={fill=green!70!black, circle, inner sep=1.5pt},
        bluepoint/.style={fill=blue!70!cyan, circle, inner sep=1.5pt, text=blue!70!cyan},
         blackpoint/.style={fill=black!70!black, circle, inner sep=1.5pt, text=blue!70!cyan},
        yellowpoint/.style={fill=yellow!80!orange, circle, inner sep=1.5pt},
        redpoint/.style={fill=red!80!red, circle, inner sep=1.5pt},
        orangepoint/.style={fill=orange!80!red, circle, inner sep=1.5pt},
        bluestar/.style={star, star points=5, star point ratio=2.5, fill=blue!70!cyan, inner sep=0.5pt, minimum size=10pt},
        yellowstar/.style={star, star points=5, star point ratio=2.5, fill=yellow!80!orange, inner sep=0.5pt, minimum size=10pt},
        graytext/.style={text=gray!70!black},
      }

     \begin{scope}[shift={(3,0)}]
     \fill[gray!10, rounded corners=10pt] (-2.5,-1.7) rectangle (2.6,2.5);
      \draw[->] (0,0) -- (2.2,0) node[right] {$r_1$};
      \draw[->] (0,0) -- (0,2.2) node[above] {$r_2$};
    
      \draw[redline, name path=spanplane] (-2.2,2.2) -- (1.2,-1.2);
    
        \draw[dashed, gray!70!black] (1,2) -- (-2,2);
        \draw[dashed, gray!70!black] (1,2) -- (1,-1);
    
        \draw[dashed, thick, red!90!red] (1,2) -- (0.5,-0.5);
      
      \draw[orangehighlight, name path=thetaplane] (-2,2) -- (1,-1);
    
    
     \draw[<->, thick] (-0.2,-0.2) -- (0.8,-1.2)node[midway, left] {$\theta>0$};
      \draw[<->, thick] (-0.2,-0.2) -- (-2.2,1.8)node[midway, left] {$\theta < 0$};
      
      \node[redpoint] (rtheta) at (0.5,-0.5) {};
      \node[blackpoint] (rtheta0) at (0,0) {};
     
      \node[bluepoint] (rstar) at (1,2) {};
      
      \node[anchor=south] at (rstar) {$r^*$};
      \node[anchor=west] at (rtheta) {$r^\beta_\theta$};
      \node[above right=0.25pt] at (rtheta0) {$r^\beta_{\theta_0}$};
       
      \node[red!80!red, align=left] at (1.5,-1.4) {$\cC(A_{\theta_0}^{\top})$};
      
      \end{scope}
      
    \end{tikzpicture}
}
  \caption{\footnotesize{An example with 3 responses and 1-d policy parameter showing failure modes of DPO. $r^*$ is the latent reward. The \textcolor{red}{red} line denotes the linear approximation $\cC(A_{\theta_0}^\top)$ of the implicit reward manifold $\cR^\beta$. The region shaded in \textcolor{orange}{orange} represents all possible implicit reward functions that DPO can possibly project onto, depending on the relative proportion of pairwise preference counts $n_{1,2}, n_{2,3}, n_{3,1}$. If $n_{3,1}$ dominates the rest, then the projection $r^\beta_\theta$ induces a post-optimized policy parameter $\theta > 0$, leading to preference reversal and reduction of expected reward, causing DPO to fail.}}
\end{wrapfigure}
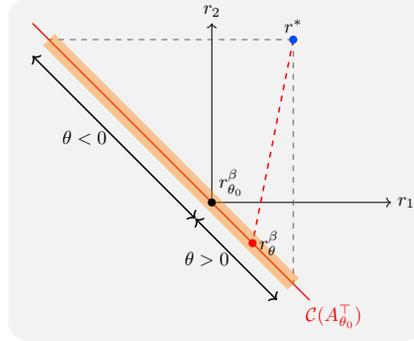

\textit{Proof.} Consider the following example with 3 responses $a_1,a_2,a_3$ with latent rewards \( r^* = [2,3,1] \), which yields the preference order $a_2\succ a_1 \succ a_3$. 
The pairwise preference counts in the dataset are highly imbalanced in a way that $n_{3,1} \gg \max\lbrace n_{1,2},n_{2,3}\rbrace$. The policy is given by $\pi_\theta = \frac{1}{Z}[e^\theta, e^{-\theta}, 1]$~, where \( Z = 1 + e^{\theta} + e^{-\theta} \). We take a uniform base policy (i.e., $\theta_0 =0$). Under the BTL model, $r^* \equiv [1,2,0]$ since both induce the same preference distribution. Hence, we will work with this equivalent $r^*$. The setting is depicted in Figure 2.

In this setting, the policy gradient matrix takes the form $A_{\theta} =
\frac{1}{Z} \left[1 + 2e^{-\theta},
-(1 + 2e^{\theta}),
e^{-\theta} - e^{\theta}\right]$. Hence $A_{\theta_0} = [1,-1,0]$ and thus $\cR^\beta \approx \Span\left([1,-1,0]\right)$ in the local neighborhood of $\theta_0 = 0$. Now, since $n_{1,3}$ dominates the other two comparison counts, the solution to the optimization problem in \eqref{eq:reverse-kl-project}, e.g.,   
\begin{align*}
    \arg\min_{r \in \cR} \sum_{i \neq j} n_{i,j} d_\KL\big(p^{\BTL}_{i,j}(r^*) || \,p^{\BTL}_{i,j}(r)\big)
\end{align*}
    will push the probability that $a_1$ is preferred over $a_3$, i.e., $p^{\BTL}_{1,3}(r^\beta_\theta)$ close to $p^{\BTL}_{1,3}(r^*)$, yielding $r^\beta_\theta(a_1)-r^\beta_\theta(a_3) = O(\alpha)$ for some $\alpha >0$. Moreover, since $r^\beta_\theta$ should lie in $\Span([1,-1,0])$, 
    DPO will end up with the reward function $r^\beta_\theta \approx [\alpha,-\alpha,0]$. This, in turn, will make $a_1$ to be preferred over $a_2$ by the learned reward, indicating a preference reversal from that given by $r^*$. Furthermore, from the equivalence relation \eqref{eq:eq:reward-policy-equiv-approx}, we get the post-optimized policy parameter $\theta = O(\alpha)$. This yields (i) $\pi_\theta(a_1) > \pi_\theta(a_2)$ (the policy favors a suboptimal response) (ii) $\pi_\theta(a_2) < \pi_{\theta_0}(a_2)$ and $\pi_\theta(a_1) > \pi_{\theta_0}(a_1)$ (likelihood of the optimal response decreases and that of a suboptimal response increases) and (iii) $\pi_{\theta}^\top r^* =\frac{e^\alpha+2e^{-\alpha}}{1+e^\alpha + e^{-\alpha}} \!<\! 1 \!=\! \pi_{\theta_0}^\top r^*$ (average reward decreases relative to the base policy).

The example exhibits the following aspects:
\begin{enumerate}[leftmargin=1em]
  \item There are pairs $(i,j)$ (e.g., $i=2, j=1$) where $a_i \succ a_j$ wrt $r^*$ (i.e., $r^*(a_i) > r^*(a_j)$) but $\pi_\theta(a_i)$ decreases and $\pi_\theta(a_j)$ increases w.r.t. the base policy. This is more extreme than likelihood displacement, where $\pi_\theta(a_i)$ and $\pi_\theta(a_j)$ both decrease or increase together but their difference $\pi_\theta(a_i) - \pi_\theta(a_j)$ is presumed to increase~\citep{razin2024unintentional, pal2024smaug}.
  \item The expected reward w.r.t. $r^*$ {\em decreases} from $\pi_{\theta_0}$, whereas two-stage RLHF would have increased the expected reward on any policy class (assuming $r^*$ is learnt accurately in the first stage).
\item Our example is based on the DPO population loss, which is effectively DPO operating in the data-rich regime where unlimited pairwise preference data from a BTL model are used as input. This circumvents the issue of failure modes of DPO known to occur due to scarce data sampling \citep{tajwar2024preference}. We show that failure modes arise due to the inherent misspecified geometry induced by the lack of model capacity, interacting with the frequencies of pairwise sampling.
  \item Our example applies in the strongest possible `oracle' optimization model where we assume that the (population) DPO loss can be optimized {\em exactly}. Our results are not dependent on the idiosyncrasies or specifics of what algorithm is used to optimize the DPO loss (e.g., gradient descent and variants as explicitly considered in other works  \citep{razin2024unintentional, pal2024smaug}), as long as it is (near) optimal.
  \item Sensitivity of DPO w.r.t. preference data distribution: In the example, if the pairwise preference counts are such that $n_{1,2} \gg \{n_{2,3},n_{3,1}\}$, then DPO would learn the desired reward function $r^\beta_\theta \approx [-\alpha,\alpha,0]$, which would, in turn, induce a policy parameter $\theta<0$, escaping the failure modes effectively. Therefore, depending on the relative proportions of pairwise preference counts $\{n_{i,j}\}_{i,j}$, one could either get the desired result or the opposite one, as in the example, or perhaps no movement at all, making DPO sensitive to preference data sampling distribution. This sensitivity to the exact preference distribution also represents a failure mode of DPO. 
\end{enumerate}

\begin{remark}[Global coverage is not sufficient for optimality]
\cite{song2024importance} argued that a global coverage condition
$\max_{s, a} \frac{\pi_\theta(a\mid s)}{\pi_{\theta_0}(a \mid s)} \leq C$ for all $\theta \in \Real^d$ is necessary for DPO to converge to the optimal policy $\theta^*$. A sufficient condition to ensure the above is $\pi_{\theta_0}(a \mid s) \geq \frac{1}{C}$. In our example, $\pi_{\theta_0}$ satisfies this with $C=3$.
However, as we have seen, DPO could learn an unaligned reward model depending on the relative frequency of preference counts, implying that global coverage is not sufficient to ensure convergence to the optimal policy. 
\end{remark}
  
\section{Towards Mitigating DPO's Pitfalls}

We propose to address the misspecification issue of DPO by first studying the nature of the RLHF policy optimization step in a suitably local sense (assuming ideal reward learning), and then using the insights gained to encourage movement towards this solution.

\subsection{Local Geometry of RLHF Optimization} 
\label{sec:RLHFlocalview}

We locally approximate the objective $J(\theta;r^*)$ in \eqref{eq:RLHF} around the base policy $\pi_{\theta_0}$. To do so, we approximate the expected reward using a first-order Taylor series expansion and the KL penalty using a second-order Taylor series expansion: 
 \begin{align*}
\mathbb{E}_{\rho, \pi_\theta}\!\left[r^*(s,a)\right] \approx \mathbb{E}_{\rho, \pi_{\theta_0}}\!\left[r^*(s,a)\right] \!+\!(\theta\!-\!\theta_0)\!^\top \!A_{\theta_0}D_{\rho,\theta_0}r^*, D_\KL(\pi_\theta ||\pi_{\theta_0}) \approx  \frac{1}{2}(\theta\!-\!\theta_0)\!^\top\! F_{\rho,\theta_0}(\theta\!-\!\theta_0)~,
\end{align*}
where $D_{\rho,\theta_0}$ is a diagonal matrix with scaled base policy-probabilities $ \rho(s)\pi_{\theta_0}(a|s)$ in the diagonal entries, and $F_{\rho,\theta_0}= \mathbb{E}_{\rho,\pi_{\theta_0}}\left[\nabla \log\pi_{\theta_0}(a|s) \nabla \log \pi_{\theta_0}(a|s)^\top\right]$ denotes the Fisher information matrix at $\pi_{\theta_0}$~\citep{amari2016information}. 
Introducing the $d \times m$ matrix $A_{\rho,\theta_0}=A_{\theta_0}D_{\rho,\theta_0}$ containing the scaled gradients $\rho(s)\nabla \pi_{\theta_0}(a|s)$ in its columns, we arrive at a local quadratic approximation of the objective
\begin{align*}
   J(\theta; r^*)  \approx \mathbb{E}_{\rho, \pi_{\theta_0}}\!\left[r^*(s,a)\right] + (\theta-\theta_0)^\top A_{\rho,\theta_0}r^* - \frac{\beta}{2}(\theta-\theta_0)^\top F_{\rho,\theta_0}(\theta-\theta_0)~.
\end{align*}
\begin{remark}
    \label{rem:errorRLHF}
    Just as described in Remark \ref{rem:errorDPO}, the error in this local quadratic approximation of $J(\theta; r^*)$ can be controlled to a desired level of accuracy by taking $\beta$ to be sufficiently large; see Proposition \ref{prop:error}.
\end{remark}
Based on this local quadratic approximation of the RLHF objective function, we can deduce (via first-order optimality conditions) a relation that suitably generalizes~\eqref{eq:reward-policy-equiv}, between the optimal policy $\pi_{\theta^*}$ and the latent reward $r^*$, to parametric policy classes:
\begin{align}\label{eq:eq:reward-policy-equiv-approx}
 A_{\rho,\theta_0}r^*  =\beta F_{\rho,\theta_0}(\theta^*-\theta_0)
 \iff     \theta^* = \theta_0+ \frac{1}{\beta} F_{\rho,\theta_0}^{\dagger}A_{\rho,\theta_0} r^*~.
\end{align}
This has the form of a natural policy gradient update \citep{kakade2001natural}. Moreover, it partitions the set of all reward functions into equivalence classes as follows. For each policy parameter $\theta$, define 
\begin{align}\label{eq:DPO-reward-class}
    \cR^\beta_{\text{eq}}(\theta) = \lbrace  r \in \Real^m:  A_{\rho,\theta_0} r = \beta F_{\rho,\theta_0}(\theta-\theta_0)\rbrace~.
\end{align}

\begin{restatable}[Equivalence classes induced by RLHF]{lemma}{equivclass}
\label{lem:equiv}
    For a fixed $\theta\in \Real^d$, two reward vectors $r_1, r_2 \in \Real^m$ belong to $\cR^\beta_{\text{eq}}(\theta)$ if and only if they
 differ by a vector $\delta \in \cN(A_{\rho,\theta_0})$, i.e., a null space element of  $A_{\rho,\theta_0}$. 
\end{restatable}

Note that for tabular policies, the class $\cR^\beta_{\text{eq}}(\theta)$ \emph{essentially} reduces to the singleton $r^\beta_\theta$, the ``implicit reward'' of DPO, i.e., $r^\beta_\theta(s,a)=\beta \log\!\frac{\pi_{\theta}(a|s)}{\pi_{\theta_0}(a|s)}$.\footnote{More formally, $\cR^\beta_{\text{eq}}(\theta) $ contains functions of the form $r^\beta_\theta(s,a)+ \beta\log Z_\theta(s)$.} 

Interestingly, we can show that DPO's linearized reward functions $\overline r_\theta^\beta$ are in bijective correspondence with RLHF's equivalence classes, with each linearized reward function being the minimum-norm representative of its equivalence class:

\begin{restatable}[Relationship between RLHF equivalence classes and DPO linearization]{proposition}{imprwd}
\label{lem:imprwd}
    For a base policy $\pi_{\theta_0}$ and KL penalty $\beta > 0$,
    the reward vector $r \in \cR^\beta_{\text{eq}}(\theta)$ which has the minimum Mahalonobis-norm $\|r\|_{D_{\rho,\theta_0}}\bydef \sqrt{r^\top D_{\rho,\theta_0} r}$,
i.e., $\argmin_{r \in \Real^m} \|r\|_{D_{\rho,\theta_0}} \quad \text{such that} \quad  A_{\rho,\theta_0} r = \beta F_{\rho,\theta_0}(\theta-\theta_0)$, is given by $\overline r^\beta_\theta=\beta A_{\theta_0}^{\top}(\theta-\theta_0)$, the local linearization of $r^\beta_\theta(s,a)=\beta \log\!\frac{\pi_{\theta}(a|s)}{\pi_{\theta_0}(a|s)}$. 
\end{restatable}

The following result helps to justify the local linear approximations made to (i) the DPO implicit reward function manifold and (ii) the RLHF policy optimization objective, by showing that error between the local approximations and the original functions can be controlled to an arbitrarily prescribed level by taking the deviation parameter $\beta$ to be suitably large (so that DPO and RLHF essentially reduce to optimization over policies in a neighborhood of $\pi_{\theta_0}$).

\begin{restatable}[Approximation errors]{proposition}{error}
\label{prop:error}
    Fix $\epsilon > 0$. There exists a bounded neighborhood $\mathcal{E} \subset \Real^d$ containing $\theta_0$, a bounded set $\cR \subset \Real^m$, and $\beta_{\min} > 0$ such that 
    for every deviation parameter $\beta > \beta_{\min}$, we have (i) $r^\beta_\theta \in \cR$
    for each $\theta \in \mathcal{E}$, (ii) $r^\beta_\dpo = r^\beta_\theta$ for some $\theta \in \mathcal{E}$, 
    (iii) $r^\beta_\theta(s,a) - \beta \nabla \log \pi_{\theta_0}(a|s)^\top (\theta-\theta_0) \le \epsilon$ for each $\theta \in \mathcal{E}$, and (iv) $J(\theta; r^*)  - \left( \mathbb{E}_{\rho, \pi_{\theta_0}}\!\left[r^*(s,a)\right] + (\theta-\theta_0)^\top A_{\rho,\theta_0}r^* - \frac{\beta}{2}(\theta-\theta_0)^\top F_{\rho,\theta_0}(\theta-\theta_0) \right)  \le \epsilon$ for each $\theta \in \mathcal{E}$.
\end{restatable}

\subsection{The AuxDPO Algorithm}

We introduce a new direct alignment algorithm, AuxDPO, which leverages our insights from the analysis of the local geometry of DPO and RLHF policy optimization to mitigate the failure modes of DPO in a principled manner.

Recall that, in the general setting where the true reward function $r^*$ is misspecified (outside the implicit reward manifold $\cR^\beta$), then DPO finds the optimal policy $\theta^*$ only if the reverse-KL projection of $r^*$ on $\cR^\beta$ fortuitously lands on $r^\beta_{\theta^*}$ (Proposition \ref{prop:est}). This depends crucially on relative proportions of pairwise preference counts $\{n_{i,j}\}_{i,j}$ in the dataset and, as such, is beyond the learner's control. 

Instead, we take the following approach to encourage the optimization to move towards $r^\beta_{\theta^*}$. Note that by our local analysis of the RLHF optimization step (Sec \ref{sec:RLHFlocalview}), the true (misspecified) reward function $r^*$ and $r^\beta_{\theta^*}$ differ by an element of the nullspace of $A_{\rho,\theta_0}$, i.e., $r^*= r^\beta_{\theta^*} + \delta$, where $\delta \in \cN(A_{\rho,\theta_0}) \subsetneq \Real^m$. Therefore, if we allow the search in the reward space $\Real^m$ to utilize additional degrees of freedom $\delta$ along this nullspace (in addition to the usual degrees of freedom via $\theta$), then $r^*$ is no longer misspecified in this augmented representation. This should ideally result in the variables $\theta \in \Real^d$ and $\delta \in \cN(A_{\rho,\theta_0})$ settling in a manner that achieves $r^*= r^\beta_{\theta^*} + \delta^*$.

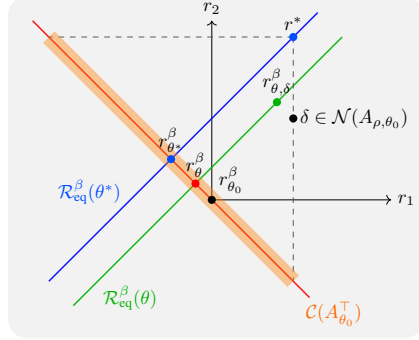
\begin{wrapfigure}[23]{r}{0.4\textwidth}
\vspace{-8pt}
    \resizebox{\linewidth}{!}{%
        \begin{tikzpicture}[scale=1.5]
  \tikzset{
    blueline/.style={thick, blue!70!blue, opacity=0.9},
    redline/.style={thick, red!70!red, opacity=0.9},
    brownline/.style={thick, brown!70!brown, opacity=0.9},
    greenline/.style={thick, green!70!black, opacity=0.9},
    yellowhighlight/.style={line width=3mm, yellow!80!orange, opacity=0.4},
    orangehighlight/.style={line width=3mm, orange!80!orange, opacity=0.4},
    greenhighlight/.style={line width=3mm, green!80!orange, opacity=0.4},
     redpoint/.style={fill=red!80!red, circle, inner sep=1.5pt},
    greenpoint/.style={fill=green!70!black, circle, inner sep=1.5pt},
    blackpoint/.style={fill=black!70!black, circle, inner sep=1.5pt},
    bluepoint/.style={fill=blue!70!cyan, circle, inner sep=1.5pt, text=blue!70!cyan},
    yellowpoint/.style={fill=yellow!80!orange, circle, inner sep=1.5pt},
    orangepoint/.style={fill=orange!80!red, circle, inner sep=1.5pt},
    bluestar/.style={star, star points=5, star point ratio=2.5, fill=blue!70!cyan, inner sep=0.5pt, minimum size=10pt},
    yellowstar/.style={star, star points=5, star point ratio=2.5, fill=yellow!80!orange, inner sep=0.5pt, minimum size=10pt},
  }


  \begin{scope}
  \definecolor{darkgreen}{rgb}{0,0.5,0}
  \fill[gray!10, rounded corners=10pt] (-2.5,-1.7) rectangle (2.6,2.5);
  \draw[->] (0,0) -- (2.2,0) node[right] {$r_1$};
  \draw[->] (0,0) -- (0,2.2) node[above] {$r_2$};
  
  \draw[redline, name path=spanplane] (-2.2,2.2) -- (1.2,-1.2);

  \draw[blueline, name path=rewardplane] (-2,-1) -- (1.3,2.3);

  \draw[greenline, name path=rewardplane] (-1.7,-1.3) -- (1.6,2);

    \draw[dashed, gray!70!black] (1,2) -- (-2,2);
    \draw[dashed, gray!70!black] (1,2) -- (1,-1);

  \draw[orangehighlight, name path=thetaplane] (-2,2) -- (1,-1);

  \node[redpoint] (rtheta) at (-0.2,0.2) {};
 
  \node[bluepoint] (rstar) at (1,2) {};
  \node[greenpoint] (rthetadelta) at (0.8,1.2) {};
  \node[blackpoint] (delta) at (1,1) {};

  \node[bluepoint] (rthetastar) at (-0.5,0.5) {};

  \node[blackpoint] (rtheta0) at (0,0) {};
  
  \node[anchor=south] at (rstar) {$r^*$};
  \node[anchor=south] at (rtheta) {$ r^\beta_\theta$};
  \node[anchor=south] at (rthetadelta) {$ r^\beta_{\theta,\delta}$};
  \node[anchor=west] at (delta) {$\delta \in \mathcal{N}(A_{\rho,\theta_0})$};
  \node[anchor=south] at (rthetastar) {$ r^\beta_{\theta^*}$};

  \node[above right=0.25pt] at (rtheta0) {$r^\beta_{\theta_0}$};

  \node[blue!70!cyan, align=center] at (-1.5,0.1) {$\cR^\beta_{\text{eq}}(\theta^*)$};

  \node[green!70!black, align=center] at (-1,-1.2) {$\cR^\beta_{\text{eq}}(\theta)$};

  \node[orange!80!red, align=left] at (1.5,-1.4) {$ \cC(A_{\theta_0}^{\top})$};
  \end{scope}
  \end{tikzpicture}
}
\caption{\footnotesize{AuxDPO fixes DPO's misspecification. $r^*$ is the latent reward. The \textcolor{blue}{blue} line denotes the equivalence class $\cR^\beta_{\text{eq}}(\theta^*)$ of all reward functions that yield the RLHF-optimal policy $\pi_{\theta^*}$. The \textcolor{red}{red} line denotes the linear approximation $\cC(A_{\theta_0}^\top)$ of the implicit reward manifold $\cR^\beta$. The region shaded in \textcolor{orange}{orange} represents all possible implicit reward functions that DPO can possibly project onto. 
The {\color{darkgreen}green} line depicts the domain of optimization over AuxDPO's auxiliary variables $\delta \in \cN(A_{\rho,\theta_0})$ for a fixed $\theta$ (the line shifts in parallel for other $\theta$). $\delta$ introduces additional degrees of freedom, which help push the KL projection of $r^*$ to lie in the equivalence class $\cR_{\theta^*}$. The projection induces the optimal policy $\pi_{\theta^*}$.}}
\end{wrapfigure}

Observe that since $A_{\rho,\theta_0} = A_{\theta_0} D_{\rho,\theta_0}$, both $\cN(A_{\theta_0})$ and $\cN(A_{\rho,\theta_0})$ have the same dimension. Moreover, $\cC(A^\top_{\theta_0})$ and $\cC(A_{\theta_0})$ also have same dimension. Thus, from rank-nullity theorem, by varying both $\theta$ and $\delta$, we can search over the entire space of the rewards $\Real^m$, contrary to DPO which searches only over $\cC(A^\top_{\theta_0})$ (under linear approximation), a manifold in $\Real^m$ with dimension at most $d$.

To this end, we introduce auxiliary variables $\delta \in \Real^m$ into the population loss of DPO~\eqref{eq:DPO-population}, and minimize it jointly over $\theta \in \Real^d, \delta$ while enforcing the nullspace constraint $\delta \in \cN(A_{\rho,\theta_0})$. This gives us the AuxDPO procedure: 
\begin{align}\label{eq:auxdpo-population}
&\minimize_{\theta \in \Real^d, \delta \in \cN(A_{\rho,\theta_0})} \cL(\theta,\delta),\, \quad \text{where}\nonumber\\
   \!\!\!\cL(\theta,\delta)  &= -\!\!\sum_{s,a,a'}\! n_{s,a,a'} \!\Big[ p_{s,a,a'}^{\BTL}(r^*) \log p_{s,a,a'}^{\BTL}(r^\beta_{\theta,\delta}) \\ &+ (1 \!-\! p_{s,a,a'}^{\BTL}(r^*)\log \big(1\!-\!p_{s,a,a'}^{\BTL}(r^\beta_{\theta,\delta})\big) \Big]\nonumber.
\end{align}
\begin{restatable}[Auxiliary variables bypass misspecification]{proposition}{estaux}
\label{prop:est-aux}
Let the hypothesis of Proposition~\ref{prop:est} hold. Fix a tolerance $\epsilon > 0$. Then, for sufficiently large $\beta > 0$, the optimization \eqref{eq:auxdpo-population} is minimized at $\theta = \theta^*$ up to error $O(\epsilon)$.
\end{restatable}
Having set up a sound AuxDPO population loss, we now develop its corresponding empirical loss version, which can be implemented with a finite preference dataset. 
We convert the constrained optimization over the set $\cN(A_{\rho,\theta_0}) \subset \Real^m$ to an unconstrained one over $\Real^m$ by adding a penalty term $\norm{\cN(A_{\rho,\theta_0})\delta}_2^2$ to the log-loss. Note that $A_{\rho,\theta_0} \delta = \mathbb{E}_{\rho,\pi_{\theta_0}}\left[\delta(s,a)\nabla \log \pi_{\theta_0}(a|s)\right]$.
Hence, we can approximate it with a given dataset $\cD=(s^{(i)},a_w^{(i)},a_l^{(i)})_{i=1}^{n}$ in Monte-Carlo fashion. 
This leads to the empirical AuxDPO loss $\cL_\cD$ over variables $\theta \in \Real^d$ and $\delta \in \Real^{2n}$:
\begin{align*}
    \cL_\cD(\theta,\delta) &= -\frac{1}{n}\sum\nolimits_{i=1}^n \log\sigma\left(r^\beta_\theta(s^{(i)},a_w^{(i)})-r^\beta_\theta(s^{(i)},a_l^{(i)})+\delta(s^{(i)}, a_w^{(i)}) - \delta(s^{(i)}, a_l^{(i)}) \right)\\
    &\!\!\!\!+ \lambda \norm{\frac{1}{2n}\sum\nolimits_{i=1}^n \left(\delta(s^{(i)}, a_w^{(i)})\nabla \log\pi_{\theta_0}(a_w^{(i)} \mid s^{(i)}) + \delta(s^{(i)}, a_l^{(i)})\nabla \log\pi_{\theta_0}(a_l^{(i)} \mid s^{(i)})\right)}^2.
\end{align*}
where $\delta = \left\lbrace \delta(s^{(i)}, a_w^{(i)}), \delta(s^{(i)}, a_l^{(i)})\right\rbrace_{i=1}^n  \in \Real^{2n}$ denotes the vector of auxiliary variables (typically $2n \ll m$) and $\lambda > 0$ is a hyper-parameter that is responsible for enforcing the nullspace constraint on $\delta$.
Note that the total number of trainable parameters is $d+2n = O(d)$, since typically, $n \ll d$.

\section{Experiments}
\label{sec:numerical}

\textbf{Datasets.} 
We conduct evaluations on two benchmark datasets: \textsc{RewardBench v2} and \textsc{MMLU-Pro}. \textsc{RewardBench v2} \citep{malik2025rewardbench}  contains $1.87K$ prompts covering categories like factuality, precise instruction following, and focus, with each prompt containing a chosen and a rejected response. \textsc{MMLU-Pro} \citep{wang2024mmlu} is a multi-task understanding dataset containing  $12K$ complex questions across various disciplines (math, law, chemistry, etc.). Each question has $10$ possible answers and a correct answer. We use \textsc{UltraFeedback} \citep{cui2024ultrafeedbackboostinglanguagemodels} as our training dataset. Specifically, we use the pre-processed and binarized version of \textsc{UltraFeedback} as presented by \citet{dong2024rlhfworkflowrewardmodeling}, which has been shown to generate higher quality reward models \citep{ ArmoRM, xiong2024iterative, banerjee2024reliablealignmentuncertaintyawarerlhf}.

\begin{table}[t]
\centering
\begin{tabular}{@{}l l l c c c c@{}}
\toprule
Model & Dataset & Method & DPO & AuxDPO & IPO & DPOP \\
\midrule
\texttt{Llama3.1-8B} & \textsc{MMLU-Pro}          & ID  & $57.14$ & $\mathbf{63.26}$ & $59.18$ & $\underline{61.22}$ \\
                     & \textsc{MMLU-Pro}          & OOD & $8.16$  & $\mathbf{14.28}$ & $\underline{10.20}$ & $6.12$ \\
                     & \textsc{RewardBench v2}    & ID  & $56.01$ & $\mathbf{66.72}$ & $61.34$ & $\underline{62.27}$ \\
                     & \textsc{RewardBench v2}    & OOD & $14.31$ & $\mathbf{32.44}$ & $\underline{20.17}$ & $19.87$ \\
\midrule
\texttt{Llama3.2-1B} & \textsc{MMLU-Pro}          & ID  & $39.58$ & $\mathbf{45.83}$ & $43.75$ & $\underline{44.21}$ \\
                     & \textsc{MMLU-Pro}          & OOD & $6.25$  & $\underline{12.52}$ & $\mathbf{14.58}$ & $4.16$ \\
                     & \textsc{RewardBench v2}    & ID  & $\underline{77.21}$ & $\mathbf{86.37}$ & $69.72$ & $71.21$ \\
                     & \textsc{RewardBench v2}    & OOD & $14.11$ & $\mathbf{43.27}$ & $\underline{20.42}$ & $18.76$ \\
\midrule
\texttt{Qwen3-0.6B}  & \textsc{MMLU-Pro}          & ID  & $53.12$ & $\mathbf{61.78}$ & $47.48$ & $\underline{56.67}$ \\
                     & \textsc{MMLU-Pro}          & OOD & $11.34$ & $\mathbf{22.22}$ & $15.56$ & $\underline{17.78}$ \\
                     & \textsc{RewardBench v2}    & ID  & $\underline{55.10}$ & $\mathbf{65.31}$ & $53.06$ & $51.02$ \\
                     & \textsc{RewardBench v2}    & OOD & $\textcolor{red}{-8.16}$ & $\mathbf{18.36}$ & $\textcolor{red}{-8.23}$ & $\underline{\textcolor{red}{-6.25}}$ \\
\bottomrule
\end{tabular}
\caption{\footnotesize Algorithm comparison. Values show percentage change in mean accuracy relative to the base policy, across \textsc{MMLU-Pro} and \textsc{RewardBench v2} under in-domain (ID) and out-of-domain (OOD) settings. Best gains are in \textbf{bold}, second-best are \underline{underlined}. Accuracies which degrade from the base policy are marked in \textcolor{red}{red}.}
\label{tab:algo_full}
\end{table}

\begin{table}[h!]
\centering
\renewcommand{\arraystretch}{1.2}
\begin{tabular}{l c ccc ccc ccc ccc}
\hline
\textbf{Subject} 
& \textbf{Base} 
& \multicolumn{2}{c}{\textbf{DPO}} 
& \multicolumn{2}{c}{\textbf{AuxDPO}} 
& \multicolumn{2}{c}{\textbf{IPO}} 
& \multicolumn{2}{c}{\textbf{DPOP}} \\
\cline{3-10}
& 
& OOD & ID 
& OOD & ID 
& OOD & ID 
& OOD & ID \\
\hline
Overall & $25.37$ & $27.06$ & $46.60$ & $\mathbf{39.26}$ & $\mathbf{51.95}$ & $31.73$ & $47.17$ & $\underline{32.64}$ & $\underline{48.25}$ \\
\midrule
Biology & $55.93$ & $\textcolor{red}{52.44}$ & $76.50$ & $\mathbf{75.07}$ & $\mathbf{86.87}$ & $\underline{60.67}$ & $79.12$ & $60.39$ & $\underline{81.50}$ \\
Business & $13.18$ & $20.03$ & $35.67$ & $\mathbf{29.96}$ & $\mathbf{43.49}$ & $\underline{24.21}$ & $32.56$ & $24.08$ & $\underline{37.96}$ \\
Chemistry & $10.42$ & $16.25$ & $\underline{37.12}$ & $\mathbf{26.78}$ & $\mathbf{40.86}$ & $21.64$ & $34.56$ & $\underline{21.73}$ & $35.86$ \\
Comp. Sc. & $24.39$ & $26.59$ & $42.78$ & $\mathbf{38.93}$ & $\mathbf{48.35}$ & $\underline{31.46}$ & $43.57$ & $31.22$ & $\underline{47.08}$ \\
Economics & $38.98$ & $\textcolor{red}{37.80}$ & $59.24$ & $\mathbf{55.12}$ & $\mathbf{68.30}$ & $\underline{44.55}$ & $62.58$ & $43.60$ & $\underline{64.50}$ \\
Engineering & $10.22$ & $17.75$ & $\underline{46.06}$ & $\mathbf{32.18}$ & $\mathbf{48.60}$ & $26.01$ & $41.46$ & $\underline{28.48}$ & $43.82$ \\
Health & $41.81$ & $\textcolor{red}{41.32}$ & $62.41$ & $\mathbf{55.67}$ & $\mathbf{68.58}$ & $\underline{44.99}$ & $64.08$  & $44.13$ & $\underline{66.23}$ \\
History & $38.06$ & $\textcolor{red}{32.55}$ & $54.96$ & $\mathbf{47.41}$ & $\mathbf{64.58}$ & $38.32$ & $\underline{59.56}$ & $\underline{40.42}$ &  $57.29$\\
Law & $27.25$ & $\textcolor{red}{23.61}$ & $42.20$ & $\mathbf{34.50}$ & $\mathbf{48.01}$ & $27.88$ & $\underline{45.25}$ & $\underline{30.43}$ & $42.31$ \\
Math & $11.77$ & $16.51$ & $28.97$ & $\mathbf{22.26}$ & $\mathbf{30.80}$ & $17.99$ & $29.34$ & $\underline{20.28}$ & $\underline{30.24}$ \\
\hline
\end{tabular}
\caption{\footnotesize Per-subject accuracies (top 10 subjects alphabetically) and overall win-rates across baseline (Llama3.1-8B) and preference optimization methods. For each method, two settings are shown: \textbf{OOD} (cross-domain transfer) and \textbf{ID} (in-domain learning) with reported results. In each row, the best accuracy is shown in \textbf{bold}, and the second-best is \underline{underlined}. Accuracies which degrade from the base policy are marked in \textcolor{red}{red}.}
\label{tab:subject_accuracy_top10}
\end{table}

\textbf{Evaluation and Methodology.} We compare AuxDPO with DPO, IPO~\citep{azar2023general}, and DPOP~\citep{pal2024smaug}.
Table~\ref{tab:algo_full} reports accuracies in terms of whether the logits of the chosen response were higher than those of the rejected response. While \textsc{RewardBench v2} provides chosen and rejected responses directly, we make \textsc{MMLU-Pro} into a preference dataset by filtering the correct answer as the chosen response and any incorrect response as the rejected response.  We consider both in-distribution (ID) and out-of-distribution (OOD) evaluation settings. In the ID setup, each dataset is split 80/20 into training and evaluation subsets, ensuring IID comparisons. In the OOD setup, models are trained on cleaned \textsc{Ultrafeedback} and evaluated on the preference datasets. We report full finetuning results on the models. Table~\ref{tab:subject_accuracy_top10} presents overall and subject-wise accuracies on \textsc{MMLU-Pro}. Accuracy is measured by comparing the finetuned model's generated answer with the correct answer provided in the dataset. We compare all methods on \texttt{Llama3.1-8B} under both OOD and ID settings. 
From both tables, we see that AuxDPO outperforms other finetuning methods across all three models. Ablation studies, implementation, and dataset details are presented in Appendix~\ref{sec:impl-auxdpo}.

\pagebreak
\bibliography{main}

\begin{thebibliography}{25}
\providecommand{\natexlab}[1]{#1}
\providecommand{\url}[1]{\texttt{#1}}
\expandafter\ifx\csname urlstyle\endcsname\relax
  \providecommand{\doi}[1]{doi: #1}\else
  \providecommand{\doi}{doi: \begingroup \urlstyle{rm}\Url}\fi

\bibitem[Amari(2016)]{amari2016information}
Shun-ichi Amari.
\newblock \emph{Information geometry and its applications}, volume 194.
\newblock Springer, 2016.

\bibitem[Azar et~al.(2023)Azar, Rowland, Piot, Guo, Calandriello, Valko, and Munos]{azar2023general}
Mohammad~Gheshlaghi Azar, Mark Rowland, Bilal Piot, Daniel Guo, Daniele Calandriello, Michal Valko, and R{\'e}mi Munos.
\newblock A general theoretical paradigm to understand learning from human preferences.
\newblock \emph{arXiv preprint arXiv:2310.12036}, 2023.

\bibitem[Banerjee and Gopalan(2024)]{banerjee2024reliablealignmentuncertaintyawarerlhf}
Debangshu Banerjee and Aditya Gopalan.
\newblock {Towards Reliable Alignment: Uncertainty-aware RLHF}, 2024.
\newblock URL \url{https://arxiv.org/abs/2410.23726}.

\bibitem[Bradley and Terry(1952)]{bradley1952rank}
Ralph~Allan Bradley and Milton~E Terry.
\newblock Rank analysis of incomplete block designs: I. the method of paired comparisons.
\newblock \emph{Biometrika}, 39\penalty0 (3/4):\penalty0 324--345, 1952.

\bibitem[Cui et~al.(2024)Cui, Yuan, Ding, Yao, He, Zhu, Ni, Xie, Xie, Lin, Liu, and Sun]{cui2024ultrafeedbackboostinglanguagemodels}
Ganqu Cui, Lifan Yuan, Ning Ding, Guanming Yao, Bingxiang He, Wei Zhu, Yuan Ni, Guotong Xie, Ruobing Xie, Yankai Lin, Zhiyuan Liu, and Maosong Sun.
\newblock {UltraFeedback: Boosting Language Models with Scaled AI Feedback}, 2024.
\newblock URL \url{https://arxiv.org/abs/2310.01377}.

\bibitem[Dong et~al.(2024)Dong, Xiong, Pang, Wang, Zhao, Zhou, Jiang, Sahoo, Xiong, and Zhang]{dong2024rlhfworkflowrewardmodeling}
Hanze Dong, Wei Xiong, Bo~Pang, Haoxiang Wang, Han Zhao, Yingbo Zhou, Nan Jiang, Doyen Sahoo, Caiming Xiong, and Tong Zhang.
\newblock {RLHF Workflow: From Reward Modeling to Online RLHF}, 2024.
\newblock URL \url{https://arxiv.org/abs/2405.07863}.

\bibitem[Gao et~al.(2024)Gao, Chang, Zhan, Oertell, Swamy, Brantley, Joachims, Bagnell, Lee, and Sun]{gao2024rebel}
Zhaolin Gao, Jonathan Chang, Wenhao Zhan, Owen Oertell, Gokul Swamy, Kiant{\'e} Brantley, Thorsten Joachims, Drew Bagnell, Jason~D Lee, and Wen Sun.
\newblock Rebel: Reinforcement learning via regressing relative rewards.
\newblock \emph{Advances in Neural Information Processing Systems}, 37:\penalty0 52354--52400, 2024.

\bibitem[Jian et~al.(2025)Jian, Yang, Ouyang, and Ye]{jian2025stable}
Chengtao Jian, Kai Yang, Ye~Ouyang, and Xiaozhou Ye.
\newblock Stable preference optimization for {LLM}s: A bilevel approach beyond direct preference optimization.
\newblock \emph{arXiv preprint arXiv:2507.07723}, 2025.

\bibitem[Kakade(2001)]{kakade2001natural}
Sham~M Kakade.
\newblock A natural policy gradient.
\newblock \emph{Advances in neural information processing systems}, 14, 2001.

\bibitem[Malik et~al.(2025)Malik, Pyatkin, Land, Morrison, Smith, Hajishirzi, and Lambert]{malik2025rewardbench}
Saumya Malik, Valentina Pyatkin, Sander Land, Jacob Morrison, Noah~A Smith, Hannaneh Hajishirzi, and Nathan Lambert.
\newblock Rewardbench 2: Advancing reward model evaluation.
\newblock \emph{arXiv preprint arXiv:2506.01937}, 2025.

\bibitem[Meng et~al.(2024)Meng, Xia, and Chen]{meng2024simpo}
Yu~Meng, Mengzhou Xia, and Danqi Chen.
\newblock Simpo: Simple preference optimization with a reference-free reward.
\newblock \emph{Advances in Neural Information Processing Systems}, 37:\penalty0 124198--124235, 2024.

\bibitem[Pal et~al.(2024)Pal, Karkhanis, Dooley, Roberts, Naidu, and White]{pal2024smaug}
Arka Pal, Deep Karkhanis, Samuel Dooley, Manley Roberts, Siddartha Naidu, and Colin White.
\newblock Smaug: Fixing failure modes of preference optimisation with dpo-positive.
\newblock \emph{arXiv preprint arXiv:2402.13228}, 2024.

\bibitem[Rafailov et~al.(2023)Rafailov, Sharma, Mitchell, Ermon, Manning, and Finn]{rafailov2023direct}
Rafael Rafailov, Archit Sharma, Eric Mitchell, Stefano Ermon, Christopher~D Manning, and Chelsea Finn.
\newblock Direct preference optimization: Your language model is secretly a reward model.
\newblock \emph{arXiv preprint arXiv:2305.18290}, 2023.

\bibitem[Razin et~al.(2024)Razin, Malladi, Bhaskar, Chen, Arora, and Hanin]{razin2024unintentional}
Noam Razin, Sadhika Malladi, Adithya Bhaskar, Danqi Chen, Sanjeev Arora, and Boris Hanin.
\newblock Unintentional unalignment: Likelihood displacement in direct preference optimization.
\newblock \emph{arXiv preprint arXiv:2410.08847}, 2024.

\bibitem[Shi et~al.(2025)Shi, Song, Zhou, Zhang, Fazel, and Du]{shi2025understanding}
Ruizhe Shi, Minhak Song, Runlong Zhou, Zihan Zhang, Maryam Fazel, and Simon~S Du.
\newblock Understanding the performance gap in preference learning: A dichotomy of {RLHF and DPO}.
\newblock \emph{arXiv preprint arXiv:2505.19770}, 2025.

\bibitem[Song et~al.(2024)Song, Swamy, Singh, Bagnell, and Sun]{song2024importance}
Yuda Song, Gokul Swamy, Aarti Singh, J~Bagnell, and Wen Sun.
\newblock The importance of online data: Understanding preference fine-tuning via coverage.
\newblock \emph{Advances in Neural Information Processing Systems}, 37:\penalty0 12243--12270, 2024.

\bibitem[Swamy et~al.(2025)Swamy, Choudhury, Sun, Wu, and Bagnell]{swamy2025all}
Gokul Swamy, Sanjiban Choudhury, Wen Sun, Zhiwei~Steven Wu, and J~Andrew Bagnell.
\newblock All roads lead to likelihood: The value of reinforcement learning in fine-tuning.
\newblock \emph{arXiv preprint arXiv:2503.01067}, 2025.

\bibitem[Tajwar et~al.(2024)Tajwar, Singh, Sharma, Rafailov, Schneider, Xie, Ermon, Finn, and Kumar]{tajwar2024preference}
Fahim Tajwar, Anikait Singh, Archit Sharma, Rafael Rafailov, Jeff Schneider, Tengyang Xie, Stefano Ermon, Chelsea Finn, and Aviral Kumar.
\newblock Preference fine-tuning of llms should leverage suboptimal, on-policy data.
\newblock \emph{arXiv preprint arXiv:2404.14367}, 2024.

\bibitem[Wang et~al.(2024{\natexlab{a}})Wang, Xiong, Xie, Zhao, and Zhang]{ArmoRM}
Haoxiang Wang, Wei Xiong, Tengyang Xie, Han Zhao, and Tong Zhang.
\newblock Interpretable preferences via multi-objective reward modeling and mixture-of-experts.
\newblock In \emph{The 2024 Conference on Empirical Methods in Natural Language Processing}, 2024{\natexlab{a}}.

\bibitem[Wang et~al.(2024{\natexlab{b}})Wang, Ma, Zhang, Ni, Chandra, Guo, Ren, Arulraj, He, Jiang, et~al.]{wang2024mmlu}
Yubo Wang, Xueguang Ma, Ge~Zhang, Yuansheng Ni, Abhranil Chandra, Shiguang Guo, Weiming Ren, Aaran Arulraj, Xuan He, Ziyan Jiang, et~al.
\newblock Mmlu-pro: A more robust and challenging multi-task language understanding benchmark.
\newblock \emph{Advances in Neural Information Processing Systems}, 37:\penalty0 95266--95290, 2024{\natexlab{b}}.

\bibitem[White(1982)]{white1982maximum}
Halbert White.
\newblock Maximum likelihood estimation of misspecified models.
\newblock \emph{Econometrica: Journal of the econometric society}, pages 1--25, 1982.

\bibitem[Xiong et~al.(2024)Xiong, Dong, Ye, Wang, Zhong, Ji, Jiang, and Zhang]{xiong2024iterative}
Wei Xiong, Hanze Dong, Chenlu Ye, Ziqi Wang, Han Zhong, Heng Ji, Nan Jiang, and Tong Zhang.
\newblock Iterative preference learning from human feedback: Bridging theory and practice for rlhf under kl-constraint.
\newblock \emph{ICML}, 2024.

\bibitem[Xu et~al.(2024{\natexlab{a}})Xu, Sharaf, Chen, Tan, Shen, Van~Durme, Murray, and Kim]{xu2024contrastive}
Haoran Xu, Amr Sharaf, Yunmo Chen, Weiting Tan, Lingfeng Shen, Benjamin Van~Durme, Kenton Murray, and Young~Jin Kim.
\newblock Contrastive preference optimization: Pushing the boundaries of llm performance in machine translation.
\newblock \emph{arXiv preprint arXiv:2401.08417}, 2024{\natexlab{a}}.

\bibitem[Xu et~al.(2024{\natexlab{b}})Xu, Fu, Gao, Ye, Liu, Mei, Wang, Yu, and Wu]{xu2024dpo}
Shusheng Xu, Wei Fu, Jiaxuan Gao, Wenjie Ye, Weilin Liu, Zhiyu Mei, Guangju Wang, Chao Yu, and Yi~Wu.
\newblock Is dpo superior to ppo for llm alignment? a comprehensive study.
\newblock \emph{arXiv preprint arXiv:2404.10719}, 2024{\natexlab{b}}.

\bibitem[Ziegler et~al.(2019)Ziegler, Stiennon, Wu, Brown, Radford, Amodei, Christiano, and Irving]{ziegler2019fine}
Daniel~M Ziegler, Nisan Stiennon, Jeffrey Wu, Tom~B Brown, Alec Radford, Dario Amodei, Paul Christiano, and Geoffrey Irving.
\newblock Fine-tuning language models from human preferences.
\newblock In \emph{arXiv preprint arXiv:1909.08593}, 2019.

\end{thebibliography}
\bibliographystyle{plainnat}
\newpage
\appendix

\section{Missing Proofs}
In this section, we restate our theoretical results and provide their proofs.

\est*
\begin{proof}
First note that minimizing $\cL(\theta)$ is equivalent to minimizing the expected negative log-likelihood, i.e.,
\begin{align*}
\min_{\theta \in \mathbb{R}^d} -\sum_{s,a,a'} n_{s,a,a'} \left[ p_{s,a,a'}^{\BTL}(r^*) \log p_{s,a,a'}^{\BTL}(r_\theta) + (1 - p_{s,a,a'}^{\BTL}(r^*)\log \big(1-p_{s,a,a'}^{\BTL}(r_\theta)\big) \right].
\end{align*}
Note that
\begin{align*}
&p_{s,a,a'}^{\BTL}(r^*) \log p_{s,a,a'}^{\BTL}(r_\theta)   + (1 - p_{s,a,a'}^{\BTL}(r^*)) \log \big(1-p_{s,a,a'}^{\BTL}(r_\theta) \big)\\
&= - d_{\mathrm{KL}}\left(p_{s,a,a'}^{\BTL}(r^*)  || p_{s,a,a'}^{\BTL}(r_\theta)\right) + H(p_{s,a,a'}^{\BTL}(r^*)),  
\end{align*}
where $d_\KL(p||q)$ denotes the KL divergence b/w two Bernoulli random variables with parameters $p,q$ and $H(p)$ denotes the entropy of a Bernoulli random variable with parameter $p$. Since $H(p_{s,a,a'}^{\BTL}(r^*))$ is constant w.r.t.\ $\theta$, the loss minimization is equivalent to minimizing the reverse KL divergence:
\[
\min_{\theta \in \mathbb{R}^d} \sum_{s,a,a'} n_{s,a,a'} \cdot d_\KL\left(p_{s,a,a'}^{\BTL}(r^*) || p_{s,a,a'}^{\BTL}(r_\theta)\right).
\]
Since $\cR = \lbrace r_\theta: \theta \in \Real^d \rbrace$, this is equivalent to solving
\[
\min_{r \in \cR} \sum_{s,a,a'} n_{s,a,a'} \cdot d_{\mathrm{KL}}\left(p_{s,a,a'}^{\BTL}(r^*) || p_{s,a,a'}^{\BTL}(r)\right)~,
\]
which completes the proof.
\end{proof}

\equivclass*
\begin{proof}
    The result holds since
    \begin{align*}
      A_{\rho,\theta_0}r_1 = A_{\rho,\theta_0}r_2 \iff A_{\rho,\theta_0}(r_1 - r_2)=0 \iff  r_1 -r_2 = \delta\in \cN( A_{\rho,\theta_0}).  
    \end{align*}
    which completes the proof.
\end{proof}

\imprwd*
\begin{proof}
To this end, a representative $\overline r_\theta$ can be the reward vector $r \in \cR_\theta$ which has the minimum Mahalonobis-norm $\|r\|_{D_{\rho,\theta_0}}$,
i.e.,
\begin{align}\label{eq:reward-mahalonobis}
   &\overline r_\theta^\beta=\argmin_{r \in \Real^m} \|r\|_{D_{\rho,\theta_0}} \quad \text{s.t.} \quad  A_{\rho,\theta_0} r = \beta F_{\rho,\theta_0}(\theta-\theta_0)~,
\end{align}
where $D_{\rho,\theta_0}$ is an $m\times m$ diagonal matrix ($m=|\cS|\cdot |\cA|$), whose diagonal entries are indexed by  $\lbrace \rho(s)\pi_{\theta_0}(a|s)\rbrace_{s,a}$.  

Using the standard Lagrange multiplier technique, one can solve \eqref{eq:reward-mahalonobis} to get
\begin{align}\label{eq:implicit-reward-vec}
    \overline r_\theta^\beta =  \beta D_{\rho,\theta_0}^{-1}A_{\rho,\theta_0}^\top (A_{\rho,\theta_0} D_{\rho,\theta_0}^{-1}A_{\rho,\theta_0}^\top)^{-1} F_{\rho,\theta_0} (\theta-\theta_0)= \beta D_{\rho,\theta_0}^{-1} A_{\rho,\theta_0}^\top  (\theta - \theta_0)~,
\end{align}
where the last equality is because the Fisher information matrix satisfies
\begin{align*}
 F_{\rho,\theta}&= \mathbb{E}_{s\sim \rho(\cdot),a \sim\pi_\theta(\cdot |s)}\left[\nabla \log\pi_{\theta}(a|s) \nabla \log \pi_\theta(a|s)^\top\right]\\
 &= \sum_{s,a} \frac{\rho(s)}{\pi_\theta(a|s)}\nabla \pi_\theta(a|s) \nabla \pi_\theta(a|s)^\top= A_{\rho,\theta} D_{\rho,\theta}^{-1} A_{\rho,\theta}^\top~.
\end{align*}
This completes the proof.
\end{proof}

\error*
\begin{proof}
Define $f: \Real^m \to \Real$ by $f(r):= \sum_{s, a,a'} n_{s,a,a'} \, d_{\mathrm{KL}}\big(p^{\BTL}_{s,a,a'}(r^*) || \,p^{\BTL}_{s,a,a'}(r)\big)$. Let  
$\cR := \cup_{\beta > 0} \left\{ r \in \cR^\beta: f(r) \leq f(0) \right\}$. It follows from \eqref{eq:reverse-kl-project} that for each $\beta$, $r^\beta_\dpo \in \cR$, establishing \textbf{Part (ii)}. Furthermore, $\cR$ is a bounded set since each of the sets in the union is bounded. Let $\gamma$ be a bound on the $\ell_2$-norm of any reward vector in $\cR$ (note that $\gamma$ is independent of any deviation parameter $\beta > 0$). 

\textbf{Part (i):} The function $\theta \mapsto r^\beta_\theta$ is assumed to be smooth. Also, $r^\beta_{\theta_0} = 0$. Since $\cR$ is bounded, there exists a neighborhood $\mathcal{E}$ of $\theta_0$ such that $r^\beta_\theta \in \cR$ for all $\theta \in \mathcal{E}$. 

\textbf{Part (iii):} For $\theta \in \cE$, using first-order Taylor series approximation of $r^\beta_\theta(s,a)$ around $\theta_0$, we get the approximation error
\begin{align*}
e_r(\theta):=r^\beta_\theta(s,a) - \beta \nabla \log \pi_{\theta_0}(a|s)^\top (\theta-\theta_0) = \frac{\beta}{2} (\theta-\theta_0)^\top \nabla^2 \log \pi_{\bar\theta}(a|s)(\theta-\theta_0)   
\end{align*}
for some $\bar\theta$ in the line segment joining $\theta$ and $\theta_0$.

Suppose the log-likelihood function $\log \pi_\theta(a|s)$ is 
$L$-smooth in $\theta$ for every $(s,a)$, i.e.
\[
\|\nabla^2 \log \pi_\theta(a|s)\|_{\mathrm{op}} \le L
\quad \text{for all } (s,a),\theta,
\]
where $\|\cdot\|_{\mathrm{op}}$ denotes the spectral (operator) norm. 
Then every eigenvalue of the Hessian $H_\theta(s,a) := \nabla^2 \log \pi_\theta(s,a)$
lies in the interval $[-L,L]$. In particular,
$\lambda_{\max}\!\big(H_\theta(s,a)\big) \;\le\; L$. Using this, the error can be upper-bounded as
\begin{align*}
  e_r(\theta) \le \frac{\beta L}{2}\norm{\theta-\theta_0}^2  
\end{align*}
Let $s_\theta(s,a)=\nabla_\theta\log\pi_\theta(a|s)$ be the score function. Assume $\|s_\theta(s,a)\|\le G$. Then we have
\begin{align*}
    r^\beta_\theta(s,a) \le \beta G \norm{\theta-\theta_0} + \frac{\beta L}{2}\norm{\theta-\theta_0}^2 = O(\beta \norm{\theta-\theta_0})~. 
\end{align*}
Furthermore, since $\norm{r_\theta^\beta} \le \gamma = O(1)$ for $\theta \in \cE$, it ensures that
$\norm{\theta-\theta_0} = O(1/\beta)$. This further yields the approximation error
\begin{align*}
    e_r(\theta) \le O(1/\beta) \le \epsilon~,
\end{align*}
for $\beta > \beta_{\min}$ (which depends on $\epsilon$ and constants $L,G$).

\textbf{Part (iv):} For $\theta \in \cE$, using first order Taylor series approximation of $r^\beta_\theta(s,a)$ around $\theta_0$, we get the approximation error for the expected reward
\begin{align*}
e_{\rho,r}(\theta):=\mathbb{E}_{\rho,\pi_\theta}[r^*(s,a)] - r^{*\top}\pi_{\theta_0} + (\theta-\theta_0)^\top A_{\rho,\theta_0}r^* = \frac{1}{2}(\theta-\theta_0)^\top H_{\rho,\overline\theta}^{(r)}(\theta-\theta_0)~,   \end{align*}
for some $\overline\theta$ between $\theta_0$ and $\theta$. Here
\[
H_{\rho,\overline\theta}^{(r)}
=\sum_{s,a}\rho(s)r^*(s,a)\pi_\theta(a|s)\big[
s_\theta(s,a)s_\theta(s,a)^\top + H_\theta(s,a)\big],
\]
where $s_\theta(s,a)=\nabla_\theta\log\pi_\theta(a|s)$ and $H_\theta(s,a)=\nabla_\theta^2\log\pi_\theta(a|s)$.
Assume $\|s_\theta(s,a)\|\le G$, $\|H_\theta(s,a)\|_{\mathrm{op}}\le L$, $|r^*(s,a)|\le R_{\max}$.
Then
$\|H_{\rho,\overline\theta}^{(r)}\|_{\mathrm{op}}
\le R_{\max}(G^2+L):= M_1$
yielding
\[
|e_{\rho,r}(\theta)|\le \frac{M_1}{2}\|\theta-\theta_0\|^2.
\]
Similarly, using second order Taylor series approximation of $\mathbb{E}_{\rho}
\big[D_{\mathrm{KL}}(\pi_\theta(\cdot|s)\|\pi_{\theta_0}(\cdot|s))\big]$ around $\theta_0$, we get the approximation error 
\begin{align*}
    e_{\rho,\text{kl}}(\theta) &:= \mathbb{E}_{\rho}
\big[D_{\mathrm{KL}}(\pi_\theta(\cdot|s)\|\pi_{\theta_0}(\cdot|s))\big] - \frac{1}{2}(\theta-\theta_0)^\top F_{\rho,\theta_0} (\theta-\theta_0)\\
&= \frac{1}{6}H_{\rho,\tilde \theta}^{(kl)}
[(\theta-\theta_0),(\theta-\theta_0),(\theta-\theta_0)]
\end{align*}
for some $\tilde\theta$ between $\theta_0$ and $\theta$. Assume $\|T_\theta(s,a)\|_{\mathrm{op}}\le Q$ for
$T_\theta(s,a)=\nabla_\theta^3\log\pi_\theta(a|s)$, then $\|H_{\rho, \theta}^{(kl)}\|_{\mathrm{op}}\le
G^3+C_1LG+C_2Q=:M_2$,
with modest constants $C_1,C_2$, yielding
\[
|e_{\rho,\text{kl}}(\theta)|\le \frac{M_2}{6}\|\theta-\theta_0\|^3.
\]
Therefore, we get the total error in approximating the objective $J(\theta;r^*)$:
\begin{align*}
   &|J(\theta; r^*)  - \left( \mathbb{E}_{\rho, \pi_{\theta_0}}\!\left[r^*(s,a)\right] + (\theta-\theta_0)^\top A_{\rho,\theta_0}r^* - \frac{\beta}{2}(\theta-\theta_0)^\top F_{\rho,\theta_0}(\theta-\theta_0) \right)|\\
&\le
\frac{M_1}{2}\|\theta-\theta_0\|^2
+\frac{\beta M_2}{6}\|\theta-\theta_0\|^3 = O \left( \|\theta-\theta_0\|^2 + \beta \|\theta-\theta_0\|^3 \right)= O(1/\beta^2) 
\end{align*}
since $\norm{\theta-\theta_0} = O(1/\beta)$ for $\theta \in \cE$. Thus the error $\le \epsilon$ for some $\beta > \beta_{\min}$.
\end{proof}

\estaux*

\begin{proof}
First note that minimizing the objective in~\eqref{eq:auxdpo-population} is equivalent to minimizing the reverse KL divergence, i.e.,
\begin{align*}
&\min_{\theta \in \mathbb{R}^d,  \delta \in \cN(A_{\rho,\theta_0}) } -\!\sum_{s,a,a'}\!\! n_{s,a,a'} \!\left[ p_{s,a,a'}^{\BTL}(r^*) \log p_{s,a,a'}^{\BTL}(r^\beta_{\theta,\delta}) + (1 - p_{s,a,a'}^{\BTL}(r^*)\log \big(1-p_{s,a,a'}^{\BTL}(r^\beta_{\theta,\delta})\big) \right]\\
&= \min_{\theta \in \mathbb{R}^d, \delta \in \cN(A_{\rho,\theta_0})} \sum_{s,a,a'} n_{s,a,a'} \cdot d_{\mathrm{KL}}\left(p_{s,a,a'}^{\BTL}(r^*) || p_{s,a,a'}^{\BTL}(r^\beta_{\theta,\delta})\right),
\end{align*}
where $r^\beta_{\theta,\delta}(s,a)=r_\theta^\beta(s,a)+\delta(s,a)$. Since $\lbrace r^\beta_{\theta,\delta}: \theta \in \Real^d, \delta \in \cN(A_{\rho,\theta_0})\rbrace = \Real^m$, this is equivalent to solving
\[
\min_{r \in \Real^m} \sum_{s,a,a'} n_{s,a,a'} \cdot d_{\mathrm{KL}}\left(p_{s,a,a'}^{\BTL}(r^*) || p_{s,a,a'}^{\BTL}(r)\right)~.
\]
Because KL divergence is nonnegative and equals zero iff its arguments coincide, the above objective is minimized (to zero) if and only if $p_{s,a,a'}^{\BTL}(r) = p_{s,a,a'}^{\BTL}(r^*), \,\forall s, \forall a\neq a'$.  

From Proposition~\ref{prop:error}, since $r^\beta_\theta(s,a) - \overline r^\beta_\theta(s,a)\le \epsilon$  for each $\theta \in \mathcal{E}$, where  
$\overline r^\beta_\theta \in \cC(A_{\theta_0}^\top)$, it holds that the minimizer $r = r^\beta_{\theta^*,\delta^*}$ upto an order $O(\epsilon)$.
\end{proof}

\section{Additional Experimental Details}

\subsection{Ablation Study on Trainable Parameters}
\label{app:ablation-trainable}

\begin{figure}[H]
    \centering
    \includegraphics[width=\linewidth]{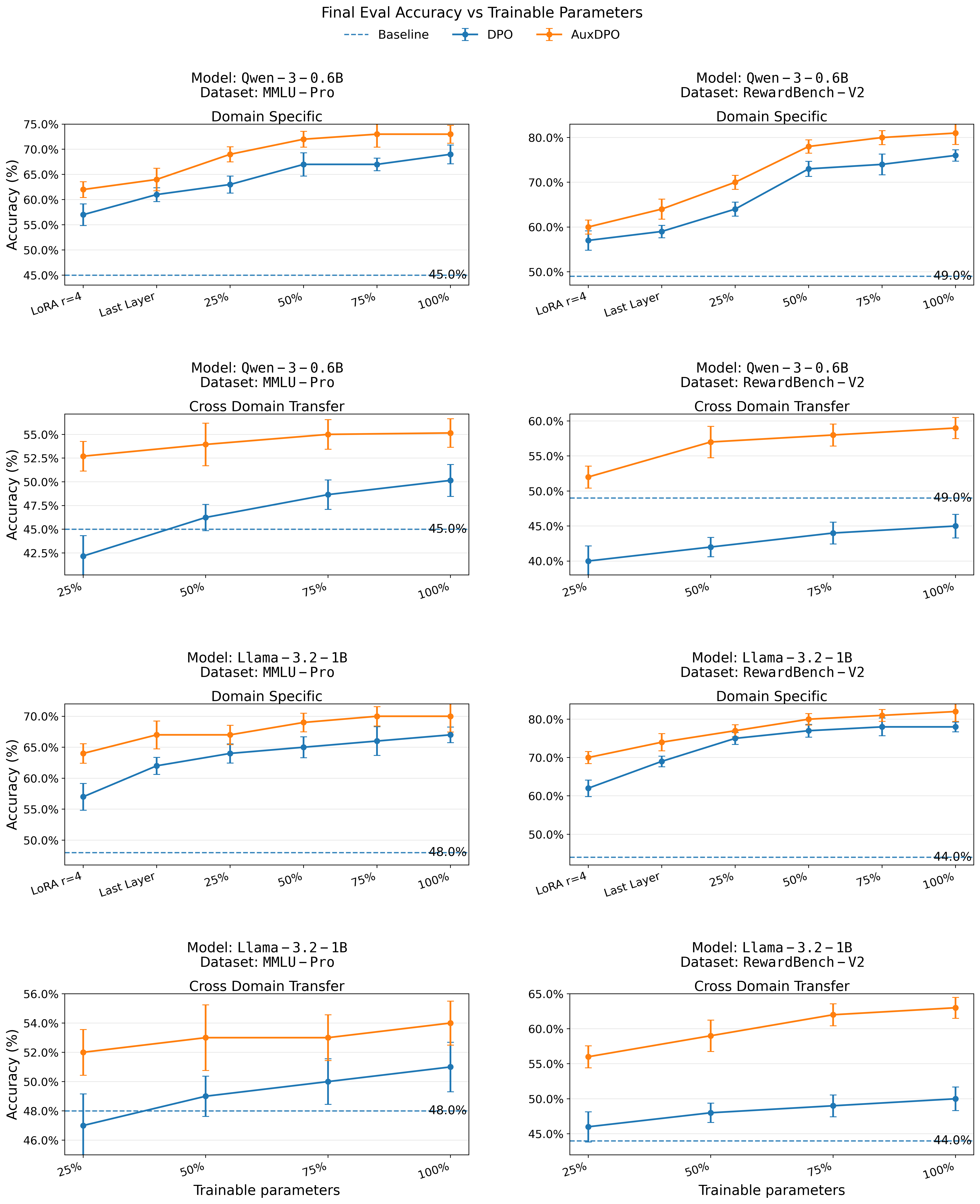}
    \caption{Final evaluation accuracy (\%) for \texttt{Qwen-3-0.6B} and \texttt{Llama-2.3-1B} on \texttt{MMLU-Pro} and \texttt{RewardBench-V2} under \textit{Domain-Specific} (ID) and \textit{Cross-Domain Transfer} (OOD) settings. Each subplot compares DPO and AuxDPO across fractions of trainable parameters (25\%, 50\%, 75\%, 100\%); for ID panels we also include \texttt{LoRA r=4} and \texttt{Last Layer} configurations. Markers show run means with error bars denoting $\pm$1 std, and dashed lines (when present) indicate per-panel baselines.}

    \label{fig:param_grid_4_by_2}
\end{figure}

In this section, we perform ablation studies on the performance of AusxDPO vs DPO by changing the number of trainable parameters. Given that AuxDPO mitigates model misspecification issues that can arise while using standard DPO, we focus mostly on having a low number of trainable parameters, where model expressibility is generally lower than entire fine-tuned models. 

\textbf{Datasets}
We conduct evaluations on two benchmark datasets: \textsc{RewardBench v2} and \textsc{MMLU-Pro}. \textsc{RewardBench v2} \citep{malik2025rewardbench} is a multi-skill reward modeling benchmark designed to bring new, challenging data for accuracy-based reward model evaluation. The dataset contains around $1.87K$ prompts covering categories as factuality, precise instruction following, focus with each prompt containing a chosen and an accepted response. \textsc{MMLU-Pro} \citep{wang2024mmlu} is a robust and challenging massive multi-task understanding dataset tailored to more rigorously benchmark large language models' capabilities. This dataset contains $12K$ complex questions across various disciplines. It's dataset contains around $12K$ questions ranging across fields such as math, law, chemistry, business, history, and psychology. For each question, there is a list of $10$ possible correct answers and the corresponding correct answer. We use \textsc{UltraFeedback} \citep{cui2024ultrafeedbackboostinglanguagemodels} as our training dataset. Specifically, we use the pre-processed and binarized version of \textsc{UltraFeedback} as presented by \citet{dong2024rlhfworkflowrewardmodeling}, which has been shown to generate higher quality reward models \citep{dong2024rlhfworkflowrewardmodeling, ArmoRM, xiong2024iterative, banerjee2024reliablealignmentuncertaintyawarerlhf}.
\textsc{RewardBench v2} is a preference dataset and provides chosen and rejected responses and is used as. \textsc{MMLU-Pro} by default is not a preference dataset. We make it into a preference dataset by filtering the correct answer as the chosen response and any incorrect response as the rejected response.

\textbf{Methodology}
We consider both in-distribution (ID) and out-of-distribution (OOD) evaluation settings. In the ID setup, each dataset is split 80/20 into training and evaluation subsets, ensuring IID comparisons. In the OOD setup, models are trained on the cleaned \textsc{UltraFeedback} dataset and evaluated on the held-out preference datasets.  We vary the fraction of trainable parameters across $25\%, 50\%, 75\%, 100\%$, by unfreezing the last $k\%$ of transformer blocks (by depth). In addition, for the in-domain (ID) panels we include two constrained-capacity baselines: \texttt{Last Layer}, where only the last transformer block is trainable and \texttt{LoRA r=4}, low-rank adapters inserted in the last transformer block on the \texttt{q\_proj}, \texttt{k\_proj}, \texttt{v\_proj}, \texttt{o\_proj} matrices. Unless otherwise noted, optimization settings, token budgets, and data splits are held fixed within each model–dataset panel so that the \emph{only} differences are (i) the algorithm (DPO, AuxDPO, IPO, DPOP) and (ii) the trainable-parameter configuration. Each marker reports the mean across 20 random seeds and error bars denote $\pm 1$ standard deviation; dashed lines indicate the per-panel base policy.

\textbf{Evaluation.}
As we decrease the number of trainable parameters, we perform finetuning by either ID or OOD method and evaluate accuracies by measuring the log probabilities of the chosen and rejected responses. Figure~\ref{fig:param_grid_4_by_2} reports final evaluation accuracy (\%) for \texttt{Qwen-3-0.6B} and \texttt{Llama-2.3-1B} across \textsc{MMLU-Pro} and \textsc{RewardBench-v2} under ID and OOD. Each subplot compares DPO vs.\ AuxDPO at \(25,50,75,100\%\) trainable parameters; ID panels additionally include \texttt{LoRA r=4} and \texttt{Last Layer}. Figure~\ref{fig:llama8_algos} reports the analogous sweep for \texttt{Llama-2.1-8B}, now comparing DPO, AuxDPO, IPO, and DPOP at the same trainable fractions (ID/OOD), with dashed baselines and \(\pm 1\) std error bars. Together, the figures isolate (a) sensitivity to optimization method, (b) sensitivity to effective capacity, and (c) ID vs.\ OOD robustness, while controlling for data and compute.

\textbf{Results.}
Within a panel, horizontal movement (left\(\rightarrow\)right) reflects increasing capacity (more parameters unfrozen); vertical separation between method curves reflects algorithmic gains at fixed capacity. The dashed line is the per-panel base policy; curves above (below) it indicate improvement (degradation). In ID panels, compare \texttt{LoRA r=4} and \texttt{Last Layer} against the fractional unfreeze settings to understand cost–performance trade-offs at very low trainable budgets. 

\textbf{Empirical trends (qualitative).}
Across models and datasets we generally observe: (i) accuracy tends to improve as the trainable fraction increases, with diminishing returns beyond \(75\%\); (ii) AuxDPO often outperforms vanilla DPO at matched capacity, particularly on \textsc{RewardBench-v2}; (iii) AuxDPO gains for OOD are more significant than than ID gains, reflecting the difficulty of cross-domain generalization for standard DPO; and (iv) At lower capacity (\texttt{LoRA r=4} and \texttt{Last Layer}) AuxDPO performs significantly better than DPO.

\begin{figure}[!htbp]
    \centering
    \includegraphics[width=\linewidth]{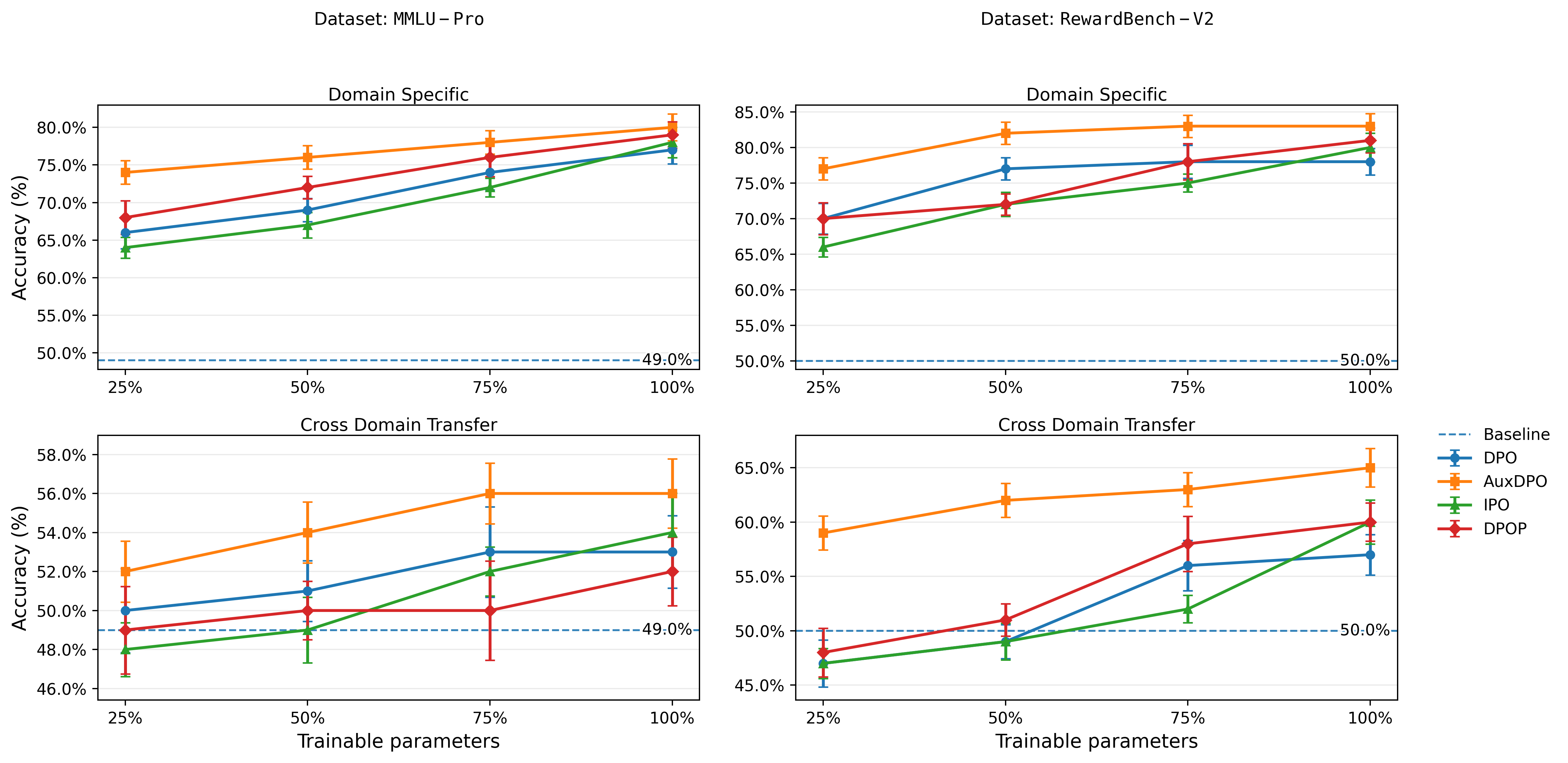}
    \caption{Final evaluation accuracy (\%) of the \texttt{Llama-2.1-8B} model on \texttt{MMLU-Pro} and \texttt{RewardBench-V2} under \textit{Domain-Specific} (ID) and \textit{Cross-Domain Transfer} (OOD) settings. Each subplot compares DPO, AuxDPO, IPO, and DPOP across fractions of trainable parameters (25\%, 50\%, 75\%, 100\%); markers show run means with error bars denoting $\pm$1 std, and a dashed line marks the per-panel baseline.}
    \label{fig:llama8_algos}
\end{figure}

\subsection{Implementation Details of AuxDPO}
\label{sec:impl-auxdpo}

\textbf{Objects and notation.}
Given a preference dataset \(\bigl(s^{(i)},a_w^{(i)},a_l^{(i)}\bigr)_{i=1}^{n}\),
AuxDPO introduces a per-example offset vector \(\delta\in\mathbb{R}^{2n}\) with
\[
\delta_{2i-1}=\delta\!\bigl(s^{(i)},a_w^{(i)}\bigr),\qquad
\delta_{2i}=\delta\!\bigl(s^{(i)},a_l^{(i)}\bigr)
\quad(i=1,\dots,n).
\]
Let \(\theta_0\) denote the reference policy and let \(d\) be the number of
\emph{trainable} parameters.
We define \(A_{\theta_0}\in\mathbb{R}^{d\times 2n}\) whose columns are
reference-model gradients on \emph{response tokens only}:
\[
A_{\theta_0}[:,2i-1]=\nabla_{\theta_0}\log \pi_{\theta_0}\!\bigl(a_w^{(i)}\mid s^{(i)}\bigr),
\qquad
A_{\theta_0}[:,2i]=\nabla_{\theta_0}\log \pi_{\theta_0}\!\bigl(a_l^{(i)}\mid s^{(i)}\bigr).
\]
The AuxDPO margin augments DPO with \(\delta\):
\[
m_i(\theta,\delta)\;=\;\beta\!\Big(
\underbrace{\log\pi_{\theta}\!\bigl(a_w^{(i)}\!\mid s^{(i)}\bigr)
-\log\pi_{\theta}\!\bigl(a_l^{(i)}\!\mid s^{(i)}\bigr)}_{\text{model}}
-\underbrace{\big[\log\pi_{\theta_0}\!\bigl(a_w^{(i)}\!\mid s^{(i)}\bigr)
-\log\pi_{\theta_0}\!\bigl(a_l^{(i)}\!\mid s^{(i)}\bigr)\big]}_{\text{reference}}
+\underbrace{\delta_{2i-1}-\delta_{2i}}_{\text{Aux term}}
\Big).
\]
The training objective is the usual DPO logistic loss
\(\sum_{i=1}^{n} \log\sigma\!\big(m_i(\theta,\delta)\big)\),
with constraints/penalties described below.

\textbf{Null-space formulation (small \(d\); e.g., LoRA/PEFT).}
When \(d\) is modest, we enforce \(A_{\theta_0}\delta=0\) exactly by optimizing in
\(\mathcal{N}(A_{\theta_0})\).
Compute an orthonormal basis
\(\Gamma=[\gamma_1,\ldots,\gamma_{2n-r}]\in\mathbb{R}^{2n\times(2n-r)}\)
for \(\mathcal{N}(A_{\theta_0})\) (e.g., thin SVD/Householder) and parameterize
\[
\delta=\Gamma c,\qquad c\in\mathbb{R}^{2n-r}.
\]
This guarantees the constraint by construction and reduces variables; we then
optimize \((\theta,c)\) jointly with the same optimizer/schedule as DPO.

\textbf{Batchwise relaxation (large \(d\)).}
For large models where global \(A_{\theta_0}\) is infeasible, we enforce the
constraint approximately per batch \(\mathcal{B}\subset\{1,\dots,n\}\) (with \(|\mathcal{B}|=B\)).
Let \(A_{\theta_0,\mathcal{B}}\) and \(\delta_{\mathcal{B}}\) denote the batch
columns/entries. We add a soft penalty and a small stabilizer:
\[
\mathcal{L}_{\text{AuxDPO}}
=\sum_{i\in\mathcal{B}} \log\sigma\!\big(m_i(\theta,\delta)\big)
\;-\;\lambda_{\text{null}}\big\|A_{\theta_0,\mathcal{B}}\delta_{\mathcal{B}}\big\|_2^2
\;+\;\lambda_{\text{amp}}\big\|\delta_{\mathcal{B}}\big\|_2^2,
\]
and maintain \(\|\delta_{\mathcal{B}}\|_2>0\) (e.g., via normalization or a small floor)
to avoid the trivial solution. This aligns \(\delta\) with \(\mathcal{N}(A_{\theta_0})\)
\emph{locally} while avoiding full-matrix costs.

\textbf{Integration details.}
We implement AuxDPO as a lightweight extension of \texttt{TRL}'s \texttt{DPOTrainer}.
A custom collator attaches the dataset index \(i\) so the trainer can address
\(\delta_{2i-1}\) and \(\delta_{2i}\).
Reference log-probabilities are computed once per batch with prompt tokens
masked (gradients over response tokens only) to build \(A_{\theta_0,\mathcal{B}}\)
on the fly in the large-\(d\) regime; in the small-\(d\) regime,
\(A_{\theta_0}\) can be precomputed at initialization.
We store and update either the per-example vector \(\delta\) directly or the
lower-dimensional coefficients \(c\) (null-space case).
All other training hyperparameters (token budgets, splits, early stopping) match
standard DPO to ensure comparability.

\textbf{What is stored.}
\emph{(i)} The per-example offsets \(\delta\in\mathbb{R}^{2n}\) (or coefficients
\(c\) in the null-space parameterization).
\emph{(ii)} Either the global \(A_{\theta_0}\) (small \(d\)) or batchwise slices
\(A_{\theta_0,\mathcal{B}}\) (large \(d\)).
Both are kept in the same precision as the model logits to minimize overhead.

\begin{longtable}{@{}p{0.25\linewidth} p{0.72\linewidth}@{}}
\caption{Implementation details for DPO / AuxDPO with an 8B policy.}
\label{tab:impl_details_appendix}\\
\toprule
\textbf{Aspect} & \textbf{Details} \\
\midrule
\endfirsthead
\toprule
\textbf{Aspect} & \textbf{Details} \\
\midrule
\endhead
\bottomrule
\endfoot

\textbf{Trainer \& Args} &
\begin{minipage}[t]{\linewidth}
\begin{itemize}[leftmargin=*,topsep=0pt,partopsep=0pt,itemsep=1pt,parsep=0pt]
  \item \texttt{AuxDPOTrainer} (TRL); baseline: \texttt{DPOTrainer}.
  \item Key \texttt{DPOConfig}: \texttt{per\_device\_train\_batch\_size=4}, \texttt{num\_train\_epochs=1}, \texttt{eval\_strategy="steps"}, \texttt{eval\_steps=500}, \texttt{logging\_strategy="steps"}, \texttt{logging\_steps=100}, \texttt{remove\_unused\_columns=False}.
\end{itemize}
\end{minipage} \\
\addlinespace[8pt]
\textbf{Precision / Optimizer} &
\begin{minipage}[t]{\linewidth}
\begin{itemize}[leftmargin=*,topsep=0pt,partopsep=0pt,itemsep=1pt,parsep=0pt]
  \item Trainable params \(\approx 8.03\)B; dtype \texttt{torch.float32} for weights.
  \item AdamW in FP32 (moments \(m,v\) FP32). Runtime \texttt{bf16=True} for compute (autocast).
\end{itemize}
\end{minipage} \\
\addlinespace[8pt]
\textbf{AuxDPO knobs} &
\begin{minipage}[t]{\linewidth}
\begin{itemize}[leftmargin=*,topsep=0pt,partopsep=0pt,itemsep=1pt,parsep=0pt]
  \item \(\lambda_{\text{null}}=1.0\) (penalty on \(\lVert A^\top \delta\rVert^2\)).
  \item \(\lambda_{\text{amp}}=0.01\) (small negative L2 on batch \(\delta\)).
  \item \texttt{delta\_cap} \(=1.0\) (tanh bound on \(|\delta|\)).
  \item \texttt{aux\_lr} \(=5\times 10^{-3}\).
\end{itemize}
\end{minipage} \\
\addlinespace[8pt]
\textbf{DPOP knobs} &
\begin{minipage}[t]{\linewidth}
\begin{itemize}[leftmargin=*,topsep=0pt,partopsep=0pt,itemsep=1pt,parsep=0pt]
  \item \(\lambda_{\text{pos}}=1.0\) (penalty encouraging higher log-likelihood on chosen responses).
\end{itemize}
\end{minipage} \\
\addlinespace[8pt]
\textbf{Reference policy} &
\begin{minipage}[t]{\linewidth}
\begin{itemize}[leftmargin=*,topsep=0pt,partopsep=0pt,itemsep=1pt,parsep=0pt]
  \item \texttt{ref\_model=None} (TRL snapshots and freezes a copy).
  \item Used forward-only for the reference model passes.
\end{itemize}
\end{minipage} \\
\addlinespace[8pt]
\textbf{Parameter / state sizes (theoretical)} &
\begin{minipage}[t]{\linewidth}
\begin{itemize}[leftmargin=*,topsep=0pt,partopsep=0pt,itemsep=1pt,parsep=0pt]
  \item Weights \(\approx 29.9\,\)GiB (8.03B\(\times\)4\,B).
  \item Gradients \(\approx 29.9\,\)GiB.
  \item Adam moments \((m,v)\approx 59.8\,\)GiB.
  \item Total model states \(\approx 119.6\,\)GiB per full replica (excl.\ activations).
\end{itemize}
\end{minipage} \\
\addlinespace[8pt]
\textbf{Hardware / GPUs} &
\begin{minipage}[t]{\linewidth}
\begin{itemize}[leftmargin=*,topsep=0pt,partopsep=0pt,itemsep=1pt,parsep=0pt]
  \item \(8\times\) GPUs, \(\approx 192\,\)GiB VRAM each (ROCm total \(206{,}141{,}652{,}992\) B).
  \item Single Python process with contexts on all 8 devices.
\end{itemize}
\end{minipage} \\
\addlinespace[8pt]
\textbf{VRAM (AuxDPO)} &
\begin{minipage}[t]{\linewidth}
\begin{itemize}[leftmargin=*,topsep=0pt,partopsep=0pt,itemsep=1pt,parsep=0pt]
  \item \texttt{allocated}: \(17.4\,\)GiB on GPU1--6; \(11.9\,\)GiB on GPU0,7.
  \item \texttt{reserved}: \(81.2\,\)GiB on GPU1--6; \(54\)--\(55\,\)GiB on GPU0,7.
  \item Device-used (ROCm SMI): \(82.4\,\)GiB on GPU1--6; \(55\)--\(56\,\)GiB on GPU0,7.
\end{itemize}
\end{minipage} \\
\addlinespace[8pt]
\textbf{VRAM (DPO)} &
\begin{minipage}[t]{\linewidth}
\begin{itemize}[leftmargin=*,topsep=0pt,partopsep=0pt,itemsep=1pt,parsep=0pt]
  \item \texttt{allocated}: \(16.25\,\)GiB on GPU1--6; \(11.09\,\)GiB on GPU0,7.
  \item \texttt{reserved}: \(25.94\,\)GiB on GPU1--6; \(20.24\,\)GiB on GPU0,7.
  \item Device-used (ROCm SMI): \(29.4\,\)GiB on GPU1--6; \(22\)--\(23\,\)GiB on GPU0,7.
\end{itemize}
\end{minipage} \\
\addlinespace[8pt]
\textbf{Memory instrumentation} &
\begin{minipage}[t]{\linewidth}
\begin{itemize}[leftmargin=*,topsep=0pt,partopsep=0pt,itemsep=1pt,parsep=0pt]
  \item Step callback prints \texttt{allocated}, \texttt{reserved}, per-step peaks via
        \texttt{torch.cuda.\{memory\_allocated, memory\_reserved, max\_memory\_reserved\}}.
  \item Device totals via \texttt{rocm-smi --showmeminfo vram}.
\end{itemize}
\end{minipage} \\

\end{longtable}

\subsection{Dataset Description}
\label{sec:datasets}

\textbf{MMLU-Pro.}
\textsc{MMLU-Pro}~\citep{wang2024mmlu} is a strengthened variant of the Massive Multitask Language Understanding benchmark, designed to stress-test reasoning in large language models. It contains approximately \(12{,}000\) multiple-choice questions across 14 subjects (e.g., mathematics, law, chemistry, business, history, psychology). Each question offers 10 candidate answers with a single correct label—expanding the option set from four to ten—to curb guessing and sharpen separation among models. Compared with the original MMLU, prior work reports substantially lower accuracies and reduced sensitivity to prompt style, making \textsc{MMLU-Pro} a robust, reasoning-centric evaluation suite.

\textbf{RewardBench v2.}
\textsc{RewardBench~v2}~\citep{malik2025rewardbench} is a second-generation, multi-skill benchmark for accuracy-based evaluation of reward models on unseen human data. It comprises \(\sim\!1{,}870\) prompts spanning six subsets—\emph{Factuality}, \emph{Precise Instruction Following}, \emph{Math}, \emph{Safety}, \emph{Focus}, and \emph{Ties}. Each prompt includes a preferred (“chosen”) response and multiple rejected responses. Accuracy is measured by whether the reward model assigns a higher score to the chosen response than to all rejected alternatives. Compared to v1, the v2 release emphasizes harder, out-of-distribution prompts and reports per-subset counts to facilitate reproducible evaluation.\footnote{See the dataset card and paper for construction details, category counts, and scoring.}

\textbf{UltraFeedback.}
\textsc{UltraFeedback}~\citep{cui2024ultrafeedbackboostinglanguagemodels} is a large human-preference corpus widely used for preference optimization. We adopt the standardized pairwise format (chosen vs.\ rejected); the public \emph{preference-standard} split contains approximately \(340{,}000\) rows, supporting stable optimization and serving as our training source for out-of-distribution experiments. Following common practice for DPO-style training, we use the preprocessed, binarized release of \textsc{UltraFeedback} curated by \citet{dong2024rlhfworkflowrewardmodeling}, which has been shown to yield higher-quality reward models~\citep{dong2024rlhfworkflowrewardmodeling,ArmoRM,xiong2024iterative,banerjee2024reliablealignmentuncertaintyawarerlhf}.

\subsection{Synthetic Experiments}

We demonstrate the failure of DPO using the example of Proposition~\ref{prop:toy}. Recall that there are 3 responses with true rewards \( r^* = [1,2,0] \) and
preference ordering $a_2 \succ a_1 \succ a_3$. We take the
policy $\pi_\theta \propto [e^\theta, e^{-\theta}, 1]$. The base policy with $\theta_0 = 0$ has the average reward $\pi_{\theta_0}^\top r^*$.

 \begin{minipage}{0.45\textwidth}
\centering
\begin{tabular}{l|c|c|c}
\hline
\textbf{Method} &  $\theta$ & $\pi_\theta$ & $ \pi_{\theta}^\top r^*$ \\
\hline
DPO       & $0.40$  & $[0.47,0.21, 0.32]$ & 0.895\\
IPO       & $0.10$  & $[0.37, 0.30, 0.33]$ & 0.969\\
DPOP      & $0.10$  & $[0.37, 0.30, 0.33]$ & 0.969\\
AuxDPO    & $-0.50$ & $[0.19,0.51, 0.30]$ & 1.199\\
\hline
\end{tabular}
\captionof{table}{\footnotesize{Comparison of trained policies under an imbalanced pairwise-preference regime \((n_{12}=5,\ n_{23}=5, \ n_{31}=50)\).
Columns report the learned scalar parameter \(\theta\), the policy \(\pi_\theta\), and the achieved objective \( \pi_\theta^\top r^*\) (higher is better).}}
\label{tab:performance-summary}
\begin{tabular}{l|c|c|c}
\hline
\textbf{Method}  &  $\theta$ & $\pi_\theta$ & $ \pi_{\theta}^\top r^*$ \\
\hline
DPO       & $-0.43$ & $[0.21,\ 0.48,\ 0.31]$ & 1.17 \\
IPO       & $-0.10$ & $[0.30,\ 0.37,\ 0.33]$ & 1.03 \\
AuxDPO     & $-0.50$ & $[0.19,\ 0.51,\ 0.30]$ & 1.20 \\
\hline
\end{tabular}
\captionof{table}{\footnotesize{Comparison of trained policies under a balanced pairwise-preference regime \((n_{12}=10,\ n_{13}=10,\ n_{23}=10)\).
Reported are the learned \(\theta\), the resulting policy \(\pi_\theta\), and the objective \( \pi_\theta^\top r^*\) (higher is better).}}
\label{tab:performance-summary-2}
\end{minipage}%
\hfill
\begin{minipage}{0.45\textwidth}
 First, we take pairwise preferences counts $n_{12}=5$, $n_{23}=5$ and $n_{31}=50$. We tabulate the post-optimized policy and its average reward in Table~\ref{tab:performance-summary}. We compare AuxDPO with DPO, IPO, and DPOP. We see that policies output by DPO, IPO, and DPOP change the preference order to $a_1 \succ a_3 \succ a_2$, while AuxDPO is able to maintain 
the correct ordering. Moreover, the average reward of the AuxDPO policy increases compared to the base policy, whereas for others, the reward decreases, showing their failures. 
Next, in Table~\ref{tab:performance-summary-2}, we demonstrate the sensitivity of DPO w.r.t. pairwise preference counts. For preference counts $n_{12}=10, n_{23}=10, n_{31}=10,$, we see that DPO (as well as IPO) does not suffer from the failure modes and is able to increase the average reward compared to the base policy. Note here that AuxDPO does better than both DPO and IPO.
\end{minipage}

The results in both Tables together show that DPO is vulnerable to failure modes, including preference reversal, reward reduction, and is sensitive to the relative frequency of pairwise preference counts, while AuXDPO is able to overcome all these.

\end{document}